\documentclass[11pt]{article}
\pdfoutput=1
\usepackage{bm}
\usepackage{mathrsfs}
\usepackage{amsmath,amssymb,mathtools}
\usepackage{natbib}
\usepackage[usenames]{color}
\usepackage{amsthm}

\usepackage{multirow} 
\usepackage{enumitem}

\def\T{\top}

\usepackage{subcaption}

\usepackage[colorlinks,
linkcolor=red,
anchorcolor=blue,
citecolor=blue
]{hyperref}

\usepackage{mylatexstyle}
\usepackage{ulem}
\usepackage{setspace}
\usepackage[left=1in, right=1in, top=1in, bottom=1in]{geometry}

\usepackage{xcolor}

\definecolor{pinegreen}{rgb}{0.0, 0.47, 0.44}

\newcommand{\bel}{\begin{eqnarray}\label}
\newcommand{\eel}{\end{eqnarray}}
\newcommand{\bes}{\begin{eqnarray*}}
\newcommand{\ees}{\end{eqnarray*}}
\newcommand{\bei}{\begin{itemize}}
\newcommand{\beiftnt}{\begin{itemize}\footnotesize}
\newcommand{\eei}{\end{itemize}}

\def\benu{\begin{enumerate}}
\def\eenu{\end{enumerate}}

\def\complex{\mathop{{\rm I}\kern-.58em\hbox{\rm C}}\nolimits}

\def\mathbold{\boldsymbol} 

\def\ba{\mathbold{a}}

\def\bA{\mathbold{A}}

\def\bb{\mathbold{b}}

\def\bB{\mathbold{B}}

\def\bc{\mathbold{c}}

\def\bC{\mathbold{C}}

\def\bE{\mathbold{E}}

\def\bF{\mathbold{F}}

\def\bH{\mathbold{H}}

\def\bI{\mathbold{I}}

\def\bm{\mathbold{m}}

\def\bM{\mathbold{M}}

\def\bO{\mathbold{O}}

\def\bQ{\mathbold{Q}}

\def\bs{\mathbold{s}}

\def\bT{\mathbold{T}}

\def\bu{\mathbold{u}}

\def\bU{\mathbold{U}}

\def\bv{\mathbold{v}}

\def\bV{\mathbold{V}}

\def\bW{\mathbold{W}}

\def\bx{\mathbold{x}}

\def\bX{\mathbold{X}}

\def\by{\mathbold{y}}

\def\bY{\mathbold{Y}}

\def\bz{\mathbold{z}}

\def\bZ{\mathbold{Z}}

\def\bgamma{\mathbold{\gamma}}

\def\btheta{\mathbold{\theta}}

\def\bTheta{\mathbold{\Theta}}

\def\bLambda{\mathbold{\Lambda}}

\def\bmu{\mathbold{\mu}}

\def\bSigma{\mathbold{\Sigma}}

\def\cN{\mathcal{N}}

\allowdisplaybreaks

\begin{document}


\begin{titlepage}
    \begin{center}
        {\Large Nonlinear Multiple Response Regression and Learning of Latent Spaces}

        \vspace{.15in} Ye Tian,\footnotemark[1] Sanyou Wu,\footnotemark[2] and Long Feng\footnotemark[3]

        \vspace{.1in}
        \today
    \end{center}

    \footnotetext[1]{School of Computing and Data Science, The University of Hong Kong, Pokfulam Road, HKSAR, China  \\
 E-mail: ty0518@hku.hk \\.}

    \footnotetext[2]{School of Computing and Data Science, The University of Hong Kong, Pokfulam Road, HKSAR, China  \\
 E-mail: sanyouwsy@gmail.com \\.}

    \footnotetext[3]{School of Computing and Data Science, The University of Hong Kong, Pokfulam Road, HKSAR, China  \\
 E-mail: lfeng@hku.hk \\.}

    \paragraph{Abstract.}
Identifying low-dimensional latent structures within high-dimensional data has long been a central topic in the machine learning community, driven by the need for data compression, storage, transmission, and deeper data understanding. Traditional methods, such as principal component analysis (PCA) and autoencoders (AE), operate in an unsupervised manner, ignoring label information even when it is available. In this work, we introduce a unified method capable of learning latent spaces in both unsupervised and supervised settings. We formulate the problem as a nonlinear multiple-response regression within an index model context. By applying the generalized Stein's lemma, the latent space can be estimated without knowing the nonlinear link functions.
Our method can be viewed as a nonlinear generalization of PCA. Moreover, unlike AE and other neural network methods that operate as ``black boxes'', our approach not only offers better interpretability but also reduces computational complexity while providing strong theoretical guarantees. Comprehensive numerical experiments and real data analyses demonstrate the superior performance of our method. 


\end{titlepage}

\section{Introduction}
\label{sec:introduction}
    \begin{figure}[b]
\centering\includegraphics[scale=0.4]{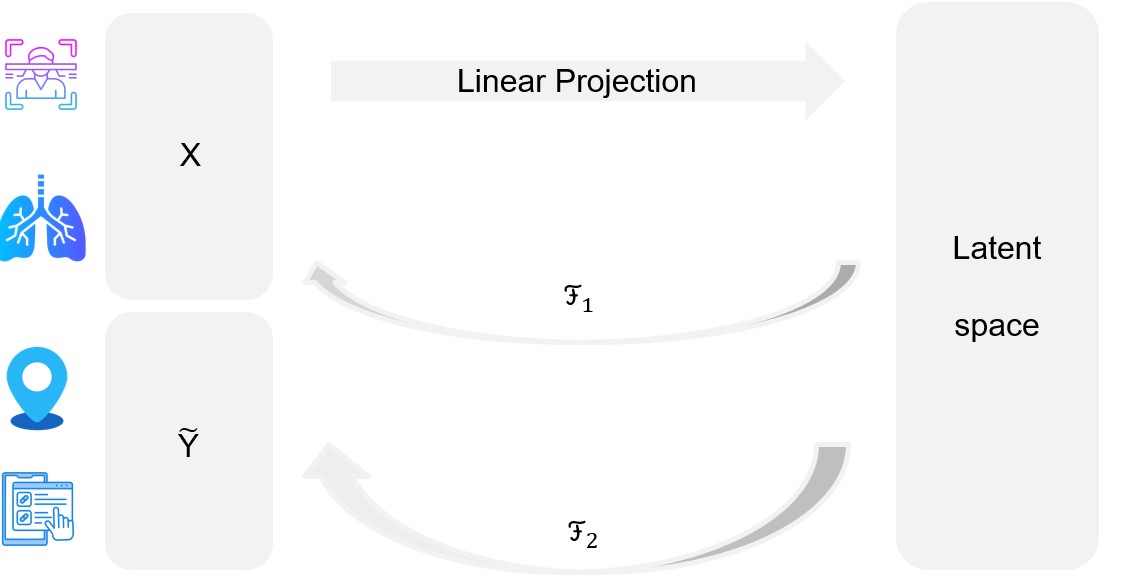}
\caption{The data we consider could be unlabeled features only, or feature-label pairs, for example,  CT and the corresponding diagnoses, or images of suspects and their positions on the surveillance screen, etc. The feature can be linearly embedded in a low-dimensional latent space, and the feature itself and possible labels can be generated from the embeddings through link functions. The goal of the work is to learn the latent space without knowing link functions.}
\end{figure}
Identifying low-dimensional latent spaces in high-dimensional data has been a longstanding focus in the statistics and machine learning community. This interest is driven not only by the goal of better understanding the data but also by practical applications such as data compression, storage, and transmission.

Principal Component Analysis (PCA, \citealp{pearson}) is arguably the simplest and most popular dimension reduction method. PCA and its extensions, such as sparse PCA \citep{Zou01062006} and robust PCA \citep{candes2009}, have found widespread applications in various fields, including computer vision, biomedical analysis, and more.
Autoencoders (AE, \citealt{vincent2008extracting}) can be viewed as a nonlinear generalization of PCA by replacing simple linear projections with complex nonlinear projections defined by neural networks. With the increasing prevalence of deep learning in recent years, AE and its extensions, such as contractive AE \citep{cae} and sparse AE \citep{makhzani2013k}, have proven effective in dimension reduction and representation learning. 

Despite their significant success, PCA and AE still encounter considerable challenges in learning the latent spaces for downstream tasks.
Firstly, both PCA and AE learn the latent space in a purely unsupervised manner. Although the low-dimensional manifold that the high-dimensional features align with reflects their inherent properties, relying solely on this information can be suboptimal when labels are available. Leveraging label information can enhance the quality of the learned latent space, making it more suitable for corresponding downstream tasks. Additionally, PCA is limited to learning the latent space in a linear setting. While autoencoders extend PCA to a nonlinear context, their ``black-box'' nature not only lacks necessary interpretability but also imposes a significant computational burden.



Rethinking the PCA, if we treat the input as the pseudo output, the data reconstruction problem can be reframed as a multiple response regression with reduced rank coefficients. 
Multiple response regression in a reduced rank setting has been extensively studied in the literature, although the majority of the work has focused on linear models. Related works include~\citet{yuan2007,mukherjee2011,chen2012rrr,lu2012convex}. With a reduced rank setup, dimension reduction can be achieved in multiple response regression, and the latent coefficient space can be learned. Beyond reduced rank regression, multiple response regression has also been considered under other settings, such as sparse matrix coefficients estimation~
\citep{LEE2012241, simon2013, wang2013block, Zou2022}. However, as these approaches are limited to linear models, complex nonlinear relationships between predictors and responses can hardly be captured and their applicability in real-world data analyses can be significantly restricted.



In the realm of nonlinear regression, the index model is a prominent statistical framework that extends linear regression to a nonlinear context. Specifically, consider the multi-index model (MIM), where the outcomes are modeled as $\EE(y|\bx) = f(\bB^{\T}\bx)$. In this formulation, $f$ represents an unknown nonlinear link function, and the matrix $\bB$ has a rank significantly smaller than the input dimension. Unlike other statistical methods that introduce nonlinearity, such as generalized linear models or spline regression, MIM inherently produces a latent embedding of the features through the introduction of $\bB$. This characteristic is particularly aligned with our objective of estimating the latent space.

Extensive research has been conducted on MIM, particularly on the estimation of the low-rank matrix $\bB$. The primary challenge lies in the fact that, without prior knowledge of the link functions, traditional methods such as least squares or maximum likelihood estimators are not applicable for estimating $\bB$. Nonetheless, the statistics community has introduced a plethora of innovative solutions over the years. For instance, in low-dimensional settings with fixed rank of $\bB$, \citet{friedman1981projection} introduced the projection pursuit regression, under which the link functions and low-rank matrix can be estimated through an alternating optimization process. 
A series of breakthroughs related to index models are proposed by Ker-Chau Li, including regression under link violation~\citep{li1989regression}, sliced inverse regression~\citep{li1991sliced}, and
principal Hessian directions~\citep{li1992principal}. 
In the high-dimensional setting where the rank of $\bB$ is allowed to grow with the sample size, the estimation of $\bB$ is considered under the sparse subspace assumption \citep{chen2010,tan2018convex}. 
 More recently, a line of research focused on using Stein's method in index models when the distribution of input is known.  For example,~\citet{yangm} proposed to learn $\bB$ in a non-Gaussian MIM through a second-order Stein’s method. Unlike previous approaches, 
 the Gaussian or elliptically symmetric assumption on the feature is not required in \citet{yangm}, thus the application of Stein's method in index models can be significantly broadened. 
Despite the success of these approaches, we shall note that most of the work concentrates on the scenario with univariate responses. To the best of our knowledge, the application of Stein's method to nonlinear settings with multiple responses has not been thoroughly explored.


In this work, we present a unified method for learning latent spaces in both unsupervised and supervised settings. We address the problem as a nonlinear multiple-response regression within the context of an index model. By leveraging the generalized Stein's lemma, our method estimates the latent space without requiring knowledge of the nonlinear link functions.
Our approach can be seen as a nonlinear extension of PCA. Furthermore, unlike neural network methods such as AEs that operate as ``black boxes'', our framework effectively manages nonlinear models without depending on gradient-based techniques. Consequently, our approach not only offers improved interpretability, but also reduces computational complexity, and provides strong theoretical guarantees.
Comprehensive numerical experiments and real data analyses demonstrate the superior performance of our method.

\textbf{Notations:} We use lowercase letters $a$, $b$ to denote scalars, bold lowercase letters $\ba$, $\bb$ to denote vectors, bold uppercase letters $\bA$, $\bB$ to denote matrices.  
For a matrix $\bA$, we let $\| \bA \|_{F}$ denote its Frobenius norm, $\sigma_{i}(\bA)$ denote the $i$-th largest singular value, $\text{SVD}_{l,i}(\bA)$ denote the top-$i$ left singular vectors corresponding to the first $i$ largest singular values, $\bA[:, :r]$ denote the sub-matrix of $\bA$ taking all rows and $1$st to $r$-th columns. For a symmetric matrix $\bSigma$, let $\bSigma^{\dagger}$ denote its pseudo-inverse, $\lambda_{i}(\bSigma)$ denote the $i$-th largest absolute values of eigenvalues and Eigen$_{i}(\bSigma)$ denote the rank-$i$ eigenvalue decomposition returning eigenvalues with first $i$ largest absolute values and corresponding top-$i$ eigenvectors or denote the top-$i$ leading eigenvectors only. 
Let $\cV_{p \times q}$ denote the set of row or column orthonormal $p\times q$ matrices, $\cV_r$ denote the set of $r\times r$ orthogonal matrices. We let $\NN_{+}$ denote positive integers. For $\forall n \in \NN_{+}$, let $[n]$ denote the set $\{1, \ldots, n\}$. For a random variable $x$, let $\|x\|_{\infty}=\sup_{\ell \in \NN_{+}} \{\EE (x^{\ell})\}^{1/\ell}$ denote its $L_{\infty}$-norm. 
In addition, given two sequences $\{x_{n}\}$ and $\{y_{n}\}$, we denote $x_{n} = \cO(y_{n})$, if $|x_{n}| \leq C|y_{n}|$
for some absolute constant $C$, denote $x_{n} = \Omega(y_{n})$ if $|x_{n}| \geq C|y_{n}|$ for some other absolute
constant $C$. We let $\nabla_{\bz}f(\bz)$ denote the derivative of $f(\bz)$ w.r.t $\bz$ and $\nabla_{\bz_{i}}f(\bz)$ the partial derivative w.r.t the $i$-th element of $\bz$. 

\section{Multiple Response Model and PCA}\label{sec:model}
\subsection{Nonlinear Multiple Response Model}
Given an input $\bx\in\mathbb{R}^p$ and an output $\by\in\mathbb{R}^q$, we consider a nonlinear multiple response model:
\begin{align}\label{model:1.2}
\by = \cF(\bB^{\top}\bx) + \bepsilon,
\end{align}
where $\bepsilon\in\mathbb{R}^q$ is a noise vector with zero-mean i.i.d. components $\epsilon_{1}, \ldots, \epsilon_{q}$, $\bB\in\mathbb{R}^{p\times r}$ is an unknown matrix,  and  $\cF=(f_{1}, \ldots, f_{q})^{\top}$ is a collection of unknown nonlinear functions with $f_{j}: \mathbb{R}^{r} \rightarrow \mathbb{R}$. We note that the nonlinear function $f_j$ can vary across different outcomes $j$. Thus, the component-wise formulation of the model~\eqref{model:1.2} is given by \begin{equation}\label{model:1} y_{j} = f_{j} (\bB^{\top}\bx) + \epsilon_{j}, \quad j \in [q]. \end{equation} 
Without loss of generality, we assume that the matrix $\bB$ has full column rank, meaning $\text{rank}(\bB) = r$. We consider a reduced rank setting where the rank $r$ is significantly smaller than $p$. 






For arbitrary nonlinear functions $f_j$, the coefficient matrix $\bB$ is unidentifiable in model~\eqref{model:1.2}. For instance, if $(\cF, \bB)$ is a solution to model~\eqref{model:1.2}, then $(\cF \circ \bO^{-\top}, \bB \bO)$ is also a solution for any invertible matrix $\bO \in \RR^{r \times r}$. Nevertheless, the column space of $\bB$ remains identifiable. Therefore, our goal is to learn the column space of the matrix $\bB$ without knowing the nonlinear functions $\cF$.



We shall emphasize that the feature-response pair $(\bx, \by)$ in model \eqref{model:1.2} is general, meaning it can include both real labels $\Tilde{\by}$ and pseudo labels $\bx$, which are the actual input features. When the response vector $\by$ corresponds solely to the input features, model~\eqref{model:1.2} transitions to the classical unsupervised setup
\begin{align}\label{model:1.3}   
\bx = \cF(\bB^{\top}\bx) + \bepsilon.
\end{align}
In this setup, the low-rank matrix $\bB$ serves a role similar to that of the leading principal components in PCA. Specifically, when $\cF$ is linear, the empirical least squares estimator of $\bB$ in model~\eqref{model:1.3} shares the same column space as the leading principal components obtained through PCA.


By allowing the response vector to be general, i.e., $\by = (\bx^{\T}, \Tilde{\by}^{\T})^{\T}$, the column space of $\bB$ can be learned by leveraging information from both the features and the labels. Given that many real-world applications have limited labeled data, learning $\bB$ in this semi-supervised manner can be especially advantageous.
On the other hand, it is important to recognize that this model essentially assumes that the pseudo output and real output share the same latent space of coefficients. While this assumption might be stringent for certain scenarios, extensive analysis of real data suggests that latent spaces estimated using this semi-supervised approach often outperform those estimated solely from features or limited labels in downstream tasks. For more details, please refer to Section~\ref{sec:rda2}.



\subsection{The Stein's Method for Latent Space Estimation}\label{sec:Stein}
In this subsection, we present two novel approaches for estimating the column space of $\bB$, utilizing both first-order and second-order Stein's methods. We begin by introducing Stein's score functions.
\begin{definition}
Let $\bx \in \mathbb{R}^{p}$ be a random vector with density $P(\bx)$. If $P(\bx)$ is differentiable, the first-order score function, denoted by $\bs(\bx): \mathbb{R}^{p} \rightarrow \mathbb{R}^{p}$, is defined as
\begin{align}\label{eq:def-fs}
        \bs(\bx) \coloneqq - \nabla_{\bx}[\ln\{P(\bx)\}] = - \nabla_{\bx}  P(\bx)  / P(\bx).
\end{align}
Moreover, let $\bs_{i}(\bx)$ denote the $i$-th element of $\bs(\bx)$ for $i=1,\ldots, p$. If $P(\bx)$ is further twice differentiable, the second-order score function $\bT(\bx): \RR^{p} \rightarrow \RR^{p \times p}$ is defined as
\begin{equation}\label{eq:def-ss}
\bT(\bx) = \nabla_{\bx}^2 P(\bx) /P(\bx) = \bs(\bx)\bs(\bx)^\T - \nabla_{\bx} \bs(\bx). 
\end{equation}
For $\forall k,l \in [p]$, let $\bT_{k,l}(\bx)$ denote the $k,l$-th entry of $\bT(\bx)$. 
\end{definition}

Under Model \eqref{model:1.2}, we have the following first-order Stein's lemma. 
\begin{lemma}(First-order Stein's Lemma)\label{lem:fs}
Suppose model~\eqref{model:1.2} holds. Assume that the expectations $\mathbb{E}\{y_{j}\bs(\bx)\}$ as well as $\mathbb{E}\{\nabla_{\bz}f_{j}(\bB^{\top}\bx)\}$ both exist and well-defined for $j \in [q]$. Further assume that $\lim_{\| \bx\| \rightarrow \infty } f_{j}(\bB^{\top}\bx) P(\bx) \rightarrow 0$.  Then, we have  
\begin{align}\label{eq:fss}
\mathbb{E}\{ y_{j} \bs(\bx)\} = \bB \mathbb{E}\{ \nabla_{\bz}f_{j} (\bB^{\top}\bx) \}.   
\end{align}
Collectively, equation~\eqref{eq:fss} suggests
\begin{align}\label{eq:fsu}
    \mathbb{E}\{ \bs(\bx) \by^{\top}   \} = \bB \bM_1, 
\end{align}
where $\bM_{1}=\mathbb{E} \{ \nabla_{\bz}f_{1} (\bB^{\top}\bx), \ldots, \nabla_{\bz}f_{q} (\bB^{\top}\bx)\}$.
\end{lemma}
Lemma~\ref{lem:fs} serves as the cornerstone of our first-order method of learning $\bB$. The key point is that by the Stein's lemma, $\bB$ can be completely split from derivatives of link functions so that we can learn $\bB$ without knowing their exact formulas.
Given $n$ samples $\{(\bx_i, \by_{i})\}^{n}_{i=1}$, we propose the following first-order estimator of $\bB$:

\begin{align}\label{eq:def-fe}
    \widehat{\bB} = \text{SVD}_{l,r}\left \{ \frac{1}{n} \sum_{i=1}^{n}  \bs(\bx_{i}) \by_{i}^{\top} \right \}.
\end{align}
We summarize our approach in Algorithm \ref{firstalg} below. Note that the score function is required in Algorithm \ref{firstalg}. When $\bs(\cdot)$ is unknown, a plug-in approach can be implemented by replacing the score function with its estimates $\hat{\bs}(\cdot)$. 
\begin{algorithm}[ht!]
\caption{First-order Method\label{firstalg}}
\begin{algorithmic}
   \STATE {\bfseries Require:} Dataset $\bx_i\in\mathbb{R}^p$, $\by_{i}\in \mathbb{R}^q$ for $i=1,\ldots,n$, rank $r$, score function $\bs(\cdot)$\\
   \quad 1: Calculate
   $\bM \leftarrow (1/n)\sum_{i=1}^{n}  \bs(\bx_{i}) \by_{i}^{\top}$\\
    \quad 2: Perform SVD:  $\left(\bU, \bSigma, \bV\right) \leftarrow \text{SVD}(\bM) $.\\
   \quad 3: {\bf return} top-$r$ left singular vectors, 
$\widehat{\bB} \leftarrow \bU[:, :r]$
\end{algorithmic}
\end{algorithm}

Beyond the first-order approach, a second-order Stein's method takes the following form.
\begin{lemma}(Second-order Stein's Lemma)\label{lem:ss} 
Suppose model~\eqref{model:1.2} holds. Assume that the expectations $\mathbb{E}\{y_{j}\bT(\bx)\}$ as well as $\mathbb{E}\{\nabla^{2}_{\bz}f_{j}(\bB^{\top}\bx)\}$ both exist and well-defined for $j \in [q]$. Further assume that $\lim_{\| \bx\| \rightarrow \infty } f_{j}(\bB^{\top}\bx) P(\bx) \rightarrow 0$ and $\lim_{\| \bx\| \rightarrow \infty } \nabla_{\bz} f_{j}(\bB^{\top}\bx) P(\bx) \rightarrow 0$.
Then, 
\begin{align}\label{eq:sss}
\mathbb{E}\{ y_{j} \bT(\bx)\} = \bB\mathbb{E}\{ \nabla_{\bz}^2f_{j} (\bB^{\top}\bx)\}\bB^\T .   
\end{align}
Collectively, equation~\eqref{eq:sss} suggests
\begin{align}\label{eq:ssu}
\mathbb{E}\left\{\frac{1}{q} \sum_{j=1}^q y_{j} \bT(\bx)\right\} = \bB\bM_2\bB^\T,
\end{align}
where $\bM_{2}=\mathbb{E}\left\{(1/q)\sum_{j=1}^q \nabla_{\bz}^2f_{j} (\bB^{\top}\bx)\right\}$.
\end{lemma}
Based on Lemma \ref{lem:ss}, we propose the following  second-order estimator:
\begin{align}\label{eq:def-se}
    \tilde{\bB} = \text{Eigen}_{r}\left \{ \frac{1}{nq} \sum_{i=1}^{n}\sum_{j=1}^q \by_{ij} \bT(\bx_{i}) \right \}. 
\end{align}
We summarize the approach in Algorithm \ref{secondalg} below. 

\begin{algorithm}[ht!]
\caption{Second-order Method\label{secondalg}}
\begin{algorithmic}
 \STATE {\bfseries Require:} Dataset $\bx_i\in\mathbb{R}^p$, $\by_{i}\in \mathbb{R}^q$ for $i=1,\ldots,n$, rank $r$, score function $\bT(\cdot)$\\
   \quad 1:  $\bM \leftarrow (nq)^{-1} \sum_{i=1}^{n}\sum_{j=1}^q \by_{i,j} \bT(\bx_{i})$\\
    \quad 2: Perform rank-$r$ Eigenvalue decomposition: $$\left(\bU, \bSigma\right )\leftarrow \text{Eigen}_r(\bM)  $$
   \quad 3: {\bf return} top-$r$ leading eigenvectors, 
$\Tilde{\bB} \leftarrow \bU$
\end{algorithmic}
\end{algorithm}

We note that both first-order and second-order approaches are applicable for arbitrary labels $\by$. This implies that our methods can be utilized in both supervised and unsupervised settings. Additionally, in scenarios with limited labeled data, we can estimate $\bB$ in a semi-supervised manner by making slight adjustments.
For instance, consider a situation where there are additional $N-n$ unlabeled samples, denoted as $\{\bx_{i}\}_{i=n+1}^{N}$, following the same marginal distribution. In such setups, $\bB$ can be estimated in a semi-supervised manner using the first-order method:
\begin{align}\label{eq:ssf}
    \text{SVD}_{l,r}\left \{ \frac{1}{n} \sum_{i=1}^{n}  \bs(\bx_{i}) \Tilde{\by}_{i}^{\top} + \frac{1}{N}  \sum_{i=1}^{N} \bs(\bx_{i}) \bx_{i}^{\top} \right \},
\end{align}
or the second-order method:
\begin{align*}
& \text{Eigen}_{r}\left \{ \frac{1}{n(q-p)} \sum_{i=1}^{n}\sum_{j=1}^{q-p} \Tilde{\by}_{ij} \bT(\bx_{i}) +\frac{1}{Np}\sum_{i=1}^{N}\sum_{j=1}^p \bx_{ij} \bT(\bx_{i}) \right \}.
\end{align*}

Beyond first and second-order methods, higher-order Stein's methods can also be applied. Related work includes \citet{janzamin2014score, balasubramanian2018}, although they do not consider multiple response settings. However, it is important to note that higher-order approaches may face additional challenges, such as increased sample size requirements and greater computational complexities. Extensive simulation studies and real data analyses presented in Section~\ref{sec:simulation} and Section~\ref{sec:rda} indicate that first and second-order methods would be sufficient in practice under general circumstances.

\subsection{Relationships to Classical Latent Space Learning Methods}

We recall the reconstruction task in model~\eqref{model:1.3}:
\begin{align*}
\bx = \cF(\bB^{\top}\bx) + \bepsilon.
\end{align*}
Under this model setup, when $\{\bx_i\}_{i=1}^n$ follows a Gaussian distribution, the following theorem indicates that the column space of $\bB$ learned using the first-order method is equivalent to that of PCA, even if $\cF$ is nonlinear.
\begin{theorem}\label{thm:pca}
Assume that model~\eqref{model:1.3} holds with $\{\bx_i\}_{i=1}^n$ following an i.i.d. Gaussian distribution with zero mean. For any arbitrary link function $\cF$, the column space learned from \eqref{eq:def-fe} matches the column space of the first $r$ leading principal components of~$\widehat\bSigma=(1/n)\sum_{i=1}^n \bx_i \bx_i^\T$.
\end{theorem}
Theorem~\ref{thm:pca} suggests that PCA can be viewed as a special case of the first-order approach with Gaussian design. For other distributions with nonlinear $\cF$, our approach offers a general framework for learning the latent space of $\bB$.  



Furthermore, model~\eqref{model:1.3} can be seen as a special case of an AutoEncoder, where the encoder is a single linear map and the decoder is a general nonlinear function. With this design, the encoder can be learned directly without needing to know the decoder. By avoiding gradient-descent-based methods for learning $\bB$, our approach is not only more computationally efficient but also comes with strong theoretical guarantees, as demonstrated in Section \ref{sec:theory}.


\section{Theoretical Analysis}\label{sec:theory}
In this section, we provide theoretical analyses of both the first-order and second-order approaches. Since only the column space of $\bB$ can be identified as discussed earlier, we use the following distance measure for two matrices $\bTheta_{1}, \bTheta_{2} \in \RR^{p \times r}$: 
\begin{align}\label{dist}
\text{dist}\left(\bTheta_{1}, \bTheta_{2}\right)=\inf_{\bV\in\mathcal{V}_r}\|\bTheta_{1}-\bTheta_{2} \bV\|_F.    
\end{align}
It should be noted that when $r=1$, this column space distance reduces to the vector distance $\|\bTheta_{1}-\bTheta_{2}\|_2$ when $\|\bTheta_{1}\|_2=\|\bTheta_{2}\|_2$. We prove that the learned column space is consistent with the true column space with the distance (\ref{dist}) converging to 0. In the following discussion, when the context is clear, we denote $\bz$ as $\bB^\T\bx$ and $\bz_{i}$ as $\bB^\T\bx_{i}$.

\subsection{Theoretical Guarantees for the First-order Method}\label{sec:fs}
To establish the convergence properties, we impose the following assumptions for the first-order method.



\begin{assumption}\label{assum:1}
   Suppose $\bx_i \in \RR^p$ are i.i.d. observations. For any $i \in [n]$, $r \in [p]$, assume that the score $\bs_k(\bx_i)$ is a sub-Gaussian random variable, i.e., $\sup_{\ell \in \NN_{+}}\ell^{-1/2}[\EE\{|\bs_{k}(\bx_{i})|^{\ell}\}]^{1/\ell} \leq C $, where $C$ is a positive constant.
\end{assumption}

\begin{assumption}\label{assum:2}
   Assume that there exists a constant $C >0$  such that $\sup_{\ell \in \NN_{+}}[\EE\{ 
   f_j(\bz)^{2\ell}\}]^{1/2\ell} \leq C$ {for all $j \in [q]$.} 
\end{assumption}
\begin{assumption}\label{assum:3}
    For $j \in [q]$ and  $ \ell \in [r] $, assume that $\EE[\nabla_{\bz_\ell}f_j(\bz)] = \Omega(1/\sqrt{r})$.
\end{assumption}

\begin{assumption}\label{assum:4}
Assume that there exists an absolute constant $C>0$ such that $\sigma_{r}(\bM_{1})\geq C\sqrt{q/r}$.
\end{assumption}

Assumption~\ref{assum:1} is a relatively mild condition and can be satisfied by a wide range of distributions, especially when $\bx$ consists of i.i.d. entries, which is widely adopted in the literature~\citep{yangs,balasubramanian2018}. It can be shown that even if $\bx$ are independent $t$, exponential such heavier-tailed distribution, $\bs(\bx)$ can still be element-wise sub-Gaussian. Moreover, it is easy to show that, if $\bx$ is multivariate normal, $\bs(\bx)$ would also be multivariate normal, therefore, each $\bs_{k}(\bx)$, $k \in [p]$, would be sub-Gaussian. It is important to note that we focus on the case of sub-Gaussian scores. Theoretical results for other types of scores, such as heavy-tailed scores, can be derived using a truncation approach. We omit the details due to space constraints.

Assumption~\ref{assum:2} requires that $f_j(\bz)$ are essentially bounded, which would be satisfied when, for instance, the link functions are bounded, $\bx$ is bounded or link functions are Lipschitz continuous and $\bx$ is sub-Gaussian.  

Assumption~\ref{assum:3} requires that link functions are sensitive enough at the true parameter $\bB$ under the distribution $\bx$. Assumption~\ref{assum:4} further requires that the non-zero singular values of  $\bM_{1}$ are of the same order. 

\begin{theorem}\label{thm:1}
   Suppose model (\ref{model:1.2}) holds. Under Assumptions \ref{assum:1}-\ref{assum:4},  the following holds with probability at least $1 - \delta$,  
    \begin{align*}                  \inf_{\bV\in\mathcal{V}_r}\|\bB-\widehat{\bB} \bV\|_F = \cO\left(\sqrt{\frac{rp}{n}}\sqrt{\ln\frac{pq}{\delta}}\right).        
    \end{align*} 
\end{theorem}
\begin{remark}
For the univariate case with $q = r = 1$,  the $L_{2}$-convergence rate  in Theorem~\ref{thm:1}  reduces to $\cO \left\{ \sqrt{p/n} \cdot \sqrt{\ln\left(p/\delta\right)}\right\}$. It matches the optimal $L_{2}$-convergence rate $\cO(\sqrt{p/n})$ up to a logarithmic factor. In fact, consider
the simplest scenario where $f$ is linear, the $L_{2}$-convergence rate of an ordinary least squares estimator of $\bB$ is $\sqrt{p/n}$. 
\end{remark}

Beyond the assumption of known score functions, we extend our analysis to the case with unknown score functions. Specifically, we assume $\bx$ follows a Gaussian distribution with mean zero and an unknown variance $\bSigma$. 
We replace $\bs(\cdot)$ in equation~\eqref{eq:def-fe} with $\hat{\bs}(\bx)$  obtained by substitute $\bSigma$ in $\bs(\cdot)$ with its moment estimator and let $\check{\bB}$ denote the corresponding estimator of $\bB$. 
To analyze the properties of $\check{\bB}$, we introduce the following assumption on the link functions. 
\begin{assumption}\label{assum:taylor}
Suppose the first-order Taylor expansions of the link functions exist and are denoted as
$f_j(\bz)=f_j(\boldsymbol{0}_{p\times r})+\langle\nabla f_j(\boldsymbol{0}_{p\times r}), \bz\rangle + h_j(\bz)\|\bz\|_2$ with $\lim_{\|\bz\|_2\rightarrow \boldsymbol{0}} h_j(\bz)=0$, for $j \in [q]$. 
Assume that the remainder term is bounded, i.e., $|h_j(\bz_{i})| \le C $ for $j \in [q]$ and $i \in [n]$, where $C > 0$ is an absolute constant.
\end{assumption}
The following theorem quantifies the discrepancy in the column space between $\check{\bB}$ and the original estimator $\widehat{\bB}$, which is dependent on the true variance $\bSigma$.

\begin{theorem}\label{thm:2} 
 Suppose model \ref{model:1.2} holds. Assume that $\bx \sim \cN(\boldsymbol{0}, \bSigma)$ with $\lambda_{\min}(\bSigma) >0 $. Then, under Assumption \ref{assum:taylor}, the following holds with the probability approaching 1: 
\begin{align*} \inf_{\check{\bV}\in\cV_r}\|\widehat{\bB}-\check{\bB}\check{\bV}\|_F = \cO_{p}\left(\sqrt{\frac{r^2p}{n}}\right).  
\end{align*}   
\end{theorem}
Theorem \ref{thm:2} demonstrates that the column space of the plug-in estimator converges to that of the truth at a rate of $\mathcal{O}\left(\sqrt{r^2p/n}\right)$. The effectiveness of this plug-in estimator, $\check{\bB}$, is further supported by numerical analyses across various distribution types, as illustrated in Section \ref{sec:simulation}.

\subsection{Theoretical Guarantees for the Second-order Method}\label{sec:second-order}
Now we prove the theoretical properties of the second-order estimator $\Tilde{\bB}$. We shall consider the case with sub-exponential second-order scores. To aid in this analysis, we make the following assumptions.

\begin{assumption}\label{ass:second-order-subexp-1} 
Assume that $T_{k,l}(\bx)$ is sub-exponential for any $k,l$, i.e., $\sup _{m \in \NN_{+}} m^{-1}\EE\{|T_{k,l}(\bx)|^m\}^{1/m} \leq C $ for certain constant $C$.
\end{assumption}

\begin{assumption}\label{ass:second-order-subexp-2}
    For $j=1,\ldots, q$,  assume that $\sup_{\ell \in \NN_{+}}[\EE\{ 
    |f_j(\bz) |^{\ell}\}]^{1/\ell} = \| f_j(\bz) \|_{\infty}  \leq C$ for certain constant $C >0$.
\end{assumption}
\begin{assumption}\label{ass:second-order-subexp-3}
    Assume that $\sup_{\ell \in \NN_{+}}\{\EE (|\epsilon|^{\ell})\}^{1/\ell} \leq C$ for certain constant $C$. 
\end{assumption}
\begin{assumption}\label{ass:second-order-subexp-4}
Assume that $\lambda_{r}(\bM_{2})\ge C_1/r$ for some constant $C_1$. Further assume that there exists an absolute constant $C_2 > 0$, such that $ \|1/q \sum^{q}_{j=1} f_j(\bz) \|_{\infty} > C_2$. 
\end{assumption}

Assumptions~\ref{ass:second-order-subexp-1} to~\ref{ass:second-order-subexp-4} serve as the second-order equivalents to Assumptions~\ref{assum:1} to~\ref{assum:4} of the first-order method. Additional discussions of these assumptions are deferred to Section~\ref{sec:app-as} in the Supplement.

\begin{theorem}\label{thm:second-order}
Suppose model~\eqref{model:1.2} holds. Under Assumptions~\ref{ass:second-order-subexp-1} to~\ref{ass:second-order-subexp-4}, the following holds for the second-order estimator with probability at least $1 - \delta$, 
\begin{align}\label{eq:sec-ub}
\inf_{\bV\in\mathcal{V}_r}\|\bB-\Tilde{\bB} \bV\|_F = \cO \left\{\frac{p r}{\sqrt{n}}\sqrt{\ln\left(\frac{p}{\delta}\right)}\right\}.
\end{align}
\end{theorem}

Theorem \ref{thm:second-order} illustrates the convergence properties of the second-order method. We shall note that the sample size $n$ must be greater than $p^2$ to achieve convergence. In contrast, a sample size of $n=\Omega(p)$ is sufficient for the first-order method, indicating that higher-order methods require a larger sample size.
For the case with unknown score functions, the convergence can be obtained using a method similar to that described for the first-order approach in Section~\ref{sec:fs}. The details are omitted due to space limitations.

\section{Simulation Studies}\label{sec:simulation} 
In this section, we conduct extensive simulation studies to assess the finite sample performances of the proposed estimators in a supervised setup. Unsupervised and semi-supervised settings will be considered in Section \ref{sec:rda}.
  
\subsection{Data Generating Process}
We consider a nonlinear model of the following form:
\begin{align*}
    \by_{i,j} = f_{j}(\bB^{\T}\bx_{i}) + \epsilon_{i,j}, \ j \in [q], i \in [n].
\end{align*}
Throughout the experiment, we fix the rank of $\bB$ to be $r$ and assume that $r$ is known. The matrix $\bB$ is generated in two steps: first, we generate $\bB_{o} \in \mathbb{R}^{p \times q}$, with entries that are i.i.d. samples from a normal distribution $\mathcal{N}(\mu_{o}, \sigma^{2}_{o})$; second, we set $\bB = \text{SVD}_{l,r}\left( \bB_{o} \right)$.
We consider three different multivariate distributions for the input vectors $\bx$:  
 multivariate normal distribution $\mathcal{N}(0, \bSigma_{\mathcal{N}})$,  multivariate hyperbolic distribution $\cH_{\chi, \psi}(0, \bSigma_{\cH})$\footnote{Multivariate hyperbolic distribution is often used in economics, with particular application in the fields of modeling financial markets and risk management. We refer to Section in the Supplement~\ref{sec:app-mhyp} for more details.} and 
multivariate t-distribution $t_{\nu}(0, \bSigma_{t})$.
Furthermore, we assume that the distributions of $\bx$ are non-degenerate, meaning the dispersion matrices $\bSigma_{\cN}$, $\bSigma_{\cH}$, and $\bSigma_{t}$ are all positive definite. The random errors $\epsilon_{i,j}$ are independently drawn from $\cN(0, \sigma_{\epsilon}^{2})$.

We consider three different mechanisms for generating the link functions. In the first case, we consider linear link functions. Specifically, we let $f_{j}(\bz) = \ba^{\top}_{j} \bz, j \in [q]$.

Then, we investigate two ways of generating nonlinear link functions. We let $f_{j}(\bz) =  \ba^{\top}_{j} m_{j}(\bz), j \in [q]$\footnote{For simplicity, we suppeose $q$ is even. }, where each $m_{j}(\cdot)$ represents a certain element-wise nonlinear function, such as $\sin(x-1)$ and $(x - 1)^{3}$. 
In the second case, for the first half of $q$ functions, we select various functions $m_{1} \ldots, m_{q/2}$. For the second half of $q$ functions, we define $m_{q/2 + j }$ as $m_{j \text{ mod } (q/2)} + m_{{j\text{ mod } (q/2)}+1}$, for $j \in [q/2]$.

Finally, we consider a generalized case of the second one. For the first half of $q$ functions, we choose different functions $m_{1} \ldots, m_{q/2}$ as in the second case. For the second half of $q$ functions $m_{q/2+j}$, $j 
    \in [q/2]$, we sample $j_{1}$ uniformly from $[q/2]$ and $j_{2}$ uniformly from $[q/2] \setminus \{j_{1}\}$, then we let $m_{q/2+j} \coloneqq m_{j_{1}} + m_{j_{2}}$. 
    
For details on the parameters and other elementary functions, please refer to Section~\ref{sec:app-pln} of the Supplement.

\subsection{Performance Comparison }\label{sec:exp-dis}
We compare the performance of our approach with that of the reduced rank regression and neural networks, which is evaluated by the distance defined in~\eqref{dist} between $\bB$ and the estimates. To ensure robustness, we report the medians of distances calculated over 100 repetitions, with sample sizes $n$ ranging from 300 to 9000 and the input dimension $p$ fixed at 30. The results from three different mechanisms for generating link functions are plotted in Figure~\ref{fig:main}.

We note that for linear link functions, $\sum^{q}_{j=1}\nabla^{2}_{\bz}f_{j}(\bz) = 0$ holds, rendering the second-order method inapplicable. Therefore, we only report the performance of the first-order method in this case. We observe that for $\bx$ following all three distributions, the performances of our first-order method and the reduced rank estimators are very similar. The neural networks perform worse due to model mis-specification.

For nonlinear link functions, our second-order method can achieve better space recovery performance with smaller distances compared to the reduced rank or neural network approach, provided the sample sizes are sufficiently large.

Our first-order method performs similarly to the reduced rank estimator for $\bx \sim \cN(0, \bSigma_{\cN})$ and $\bx \sim \cH_{\chi, \psi}(0, \bSigma_{\cH})$, while outperforms the reduced rank estimator significantly for $\bx \sim t_{\nu}(0, \bSigma_{t})$ in the second and third cases, indicating that it is robust to the nonlinearity and distribution of $\bx$.

The performance of neural network estimators does not consistently improve as sample sizes increase. In the second and third cases, neural networks are mis-specified as our link functions are ``independent'', while neural networks approximate these functions sharing the same parameters before the final linear layer. As our link functions vary significantly from each other, the approximation cannot be accurate enough, leading to poor estimates of $\bB$. In contrast, our methods estimate $\bB$ without requiring the exact nonlinear format and thus, are more robust than the neural network estimator.

For extended experiments and further discussions, please refer to Section~\ref{sec:app-ee} of Supplement.

\begin{figure}[p]
\begin{subfigure}{\textwidth}
\centering    \includegraphics[width=\textwidth,height=0.33\linewidth]{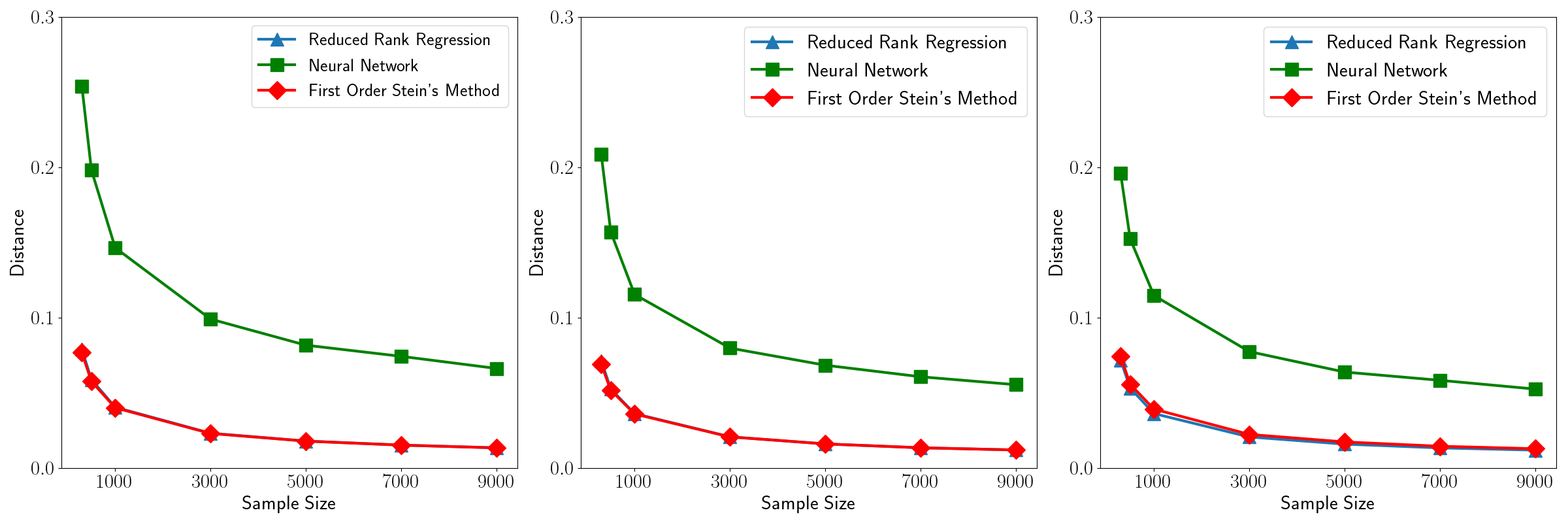}
\caption{linear link functions }
\end{subfigure}
\begin{subfigure}{\textwidth}
\centering    \includegraphics[width=\textwidth,height=0.33\linewidth]{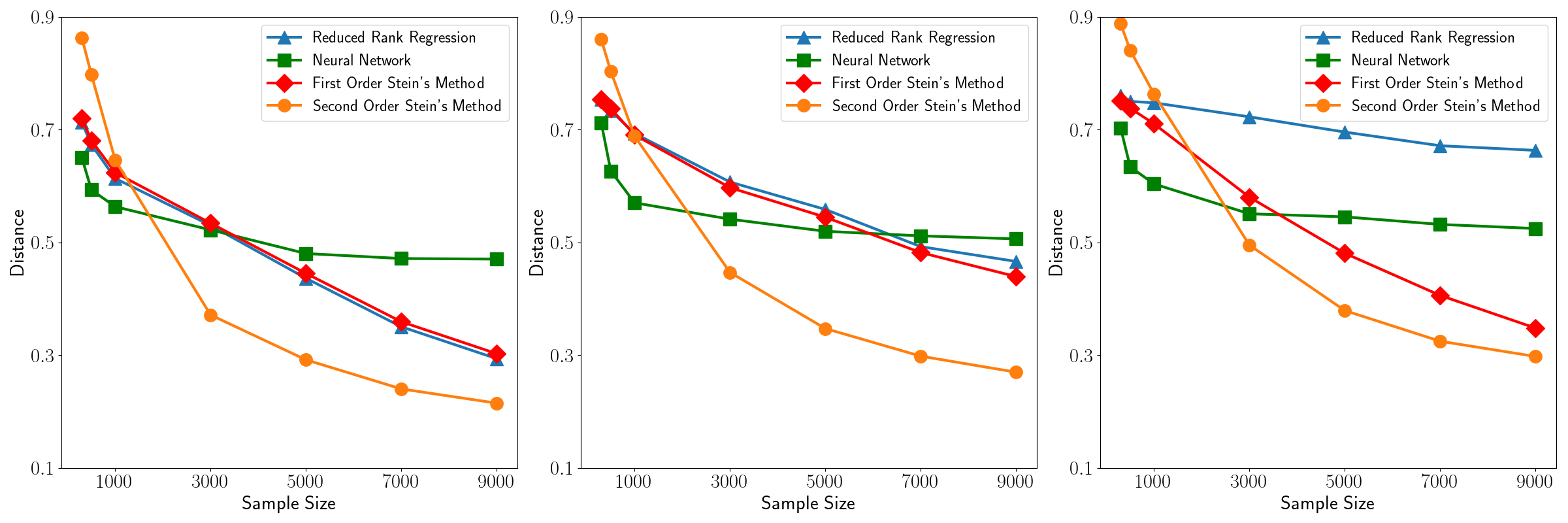}
\caption{the $1$st mechanism of generating nonlinear link functions}
\end{subfigure}
\begin{subfigure}{\textwidth}
\centering    \includegraphics[width=\textwidth,height=0.33\linewidth]{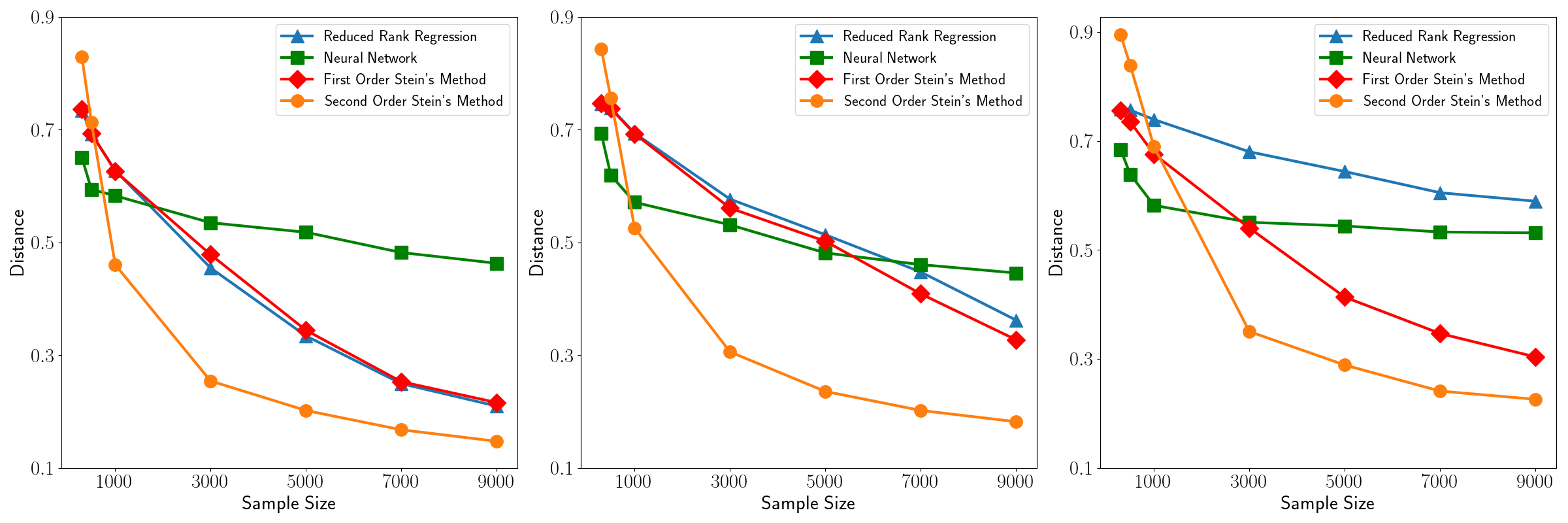}
\caption{the $2$nd mechanism of generating nonlinear link functions}
\end{subfigure}
\caption{The figure demonstrates finite sample performances of competing methods when $p = 30$. From left to right, $\bx \sim \mathcal{N}(0, \bSigma_{\mathcal{N}})$, $ \mathcal{H}_{\chi,\psi}(0, \bSigma_{\mathcal{H}})$ and $t_{\nu}(0, \bSigma_{t})$, respectively. }\label{fig:main}
\end{figure}
\section{Real Data Analysis}\label{sec:rda}
We apply  the proposed methods on two real-world datasets. On the MNIST dataset, we assess the finite sample performance of our methods in learning the latent space in model~\eqref{model:1.3}, i.e., the classical unsupervised setup. On the M1 Patch-seq dataset, we verify our conjecture in Section~\ref{sec:model} that the semi-supervised estimator would outperform those estimated solely from features or limited
labels especially in the corresponding downstream task.      

\subsection{MNIST Dataset}
In this part, we consider the reconstruction task
\begin{align*}
\bx = \cF(\bB^{\top}\bx) + \bepsilon
\end{align*}
on MNIST dataset. To make the task harder, we randomly choose $n = 10000$ samples from the original training set for training and use the full test set for testing. 

Since there is no universal metric for evaluating the quality of the estimated latent space when the truth is unknown, we provide the following three metrics. We first consider normalized root squared error (NRSE) and  structural similarity index measure (SSIM),  which measure the performances of latent embeddings in the reconstruction task. Since these measures incorporate the influence of nonlinear decoders, to eliminate such effects, we also introduce the classification accuracy of the embedding on a multi-class classification task to measure the qualities of the estimated spaces. We compare our methods with baseline methods PCA and AE. For details of definitions of the measures and applications of the methods, please refer to Section~\ref{sec:app-measures} and Section~\ref{sec:app-app-rda} of the Supplement, respectively. 

Figure \ref{fig:mes} summarizes the performances of competing methods under three different measures against embedding dimensions $h$ ranging from 10 to 18. To ensure robustness, we report the medians of these metrics over 100 repetitions. As seen in the figure, our methods consistently outperform PCA on NRSE and SSIM for embedding dimensions from 10 to 18, and on classification accuracy from 12 to 18. This is because our methods can capture nonlinear relationships among pixel values, which allows our methods to outperform the best linear approximation.

The AE performs slightly worse than our methods on NRSE and SSIM. Even though it is specifically designed for the reconstruction task, it has more parameters to estimate. Therefore, our methods perform better in cases with smaller sample sizes. For the classification task, the AE ranks last, which is also due to limited samples. Moreover, the latent space learned by the AE is specifically designed for the reconstruction task, while our estimators are task-free.

Comparing our first-order and second-order methods, they perform very close, although the second-order method is slightly inferior. One possible reason is that the multivariate normal assumption used for the second-order scores makes the second-order model mis-specified. Additionally, the relatively small sample size could also affect the performance of the second-order method. 
\begin{figure}[h]
\centering    \includegraphics[width=\textwidth]{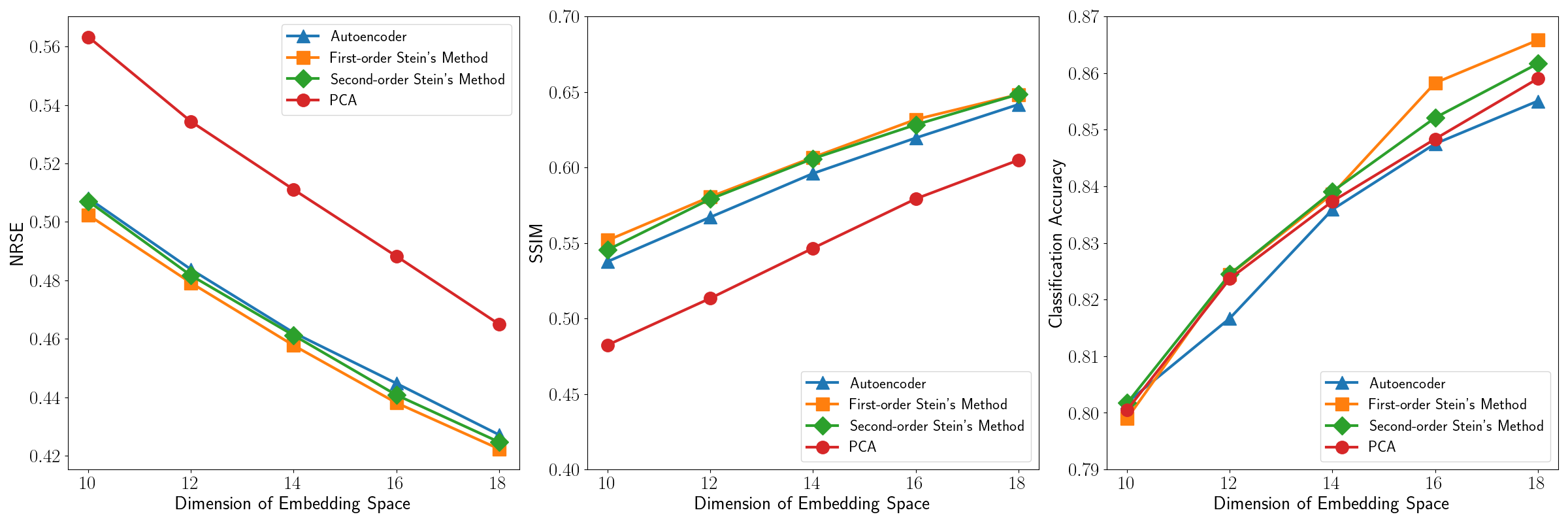}
\caption{Comparison of qualities of learned latent spaces of competing methods based on three different metrics. The values are calculated empirically on the testing set and medians over 100 repetitions are reported. From Left to right, the metrics are RMSE, SSIM and classification accuracy, respectively.}\label{fig:mes}
\end{figure}

\subsection{M1 Patch-seq
Dataset}\label{sec:rda2}
M1 dataset is one of the largest
existing Patch-seq
datasets~\citep{scala2021phenotypic}. It contains $n = 1213$ neurons from the primary motor
cortex of adult mice and spans all types of neurons, both excitatory and inhibitory. Each
cell was described by $q = 16$ electrophysiological properties and we used the $p = 1000$ most variable
genes that were selected in the original publication. We first performed sequencing depth normalization by converting the counts to counter per million. We then transformed
the data using $\log_{2}(x+1)$
transformation. Finally, we standardized all gene expression values and electrophysiological properties
to zero mean and unit variance.

We assume the general labels $\by = (\bx^{\T}, \Tilde{\by}^{\T})^{\T}$, where $\bx$ are most genes and $\Tilde{\by}$ electrophysiological properties, satisfy the model:   
\begin{align*}
\by = \cF(\bB^{\top}\bx) + \bepsilon.
\end{align*}
We separate the samples in the following steps:

1. We randomly select $n_{1}$ samples with equal probability as the testing set, which is denoted as $\{(\bx_{i}, \Tilde{\by}_{i})\}^{n_{1}}_{i = 1}$;

2. We randomly choose $n_{2}$
samples uniformly from the remaining to be the training set, denoted as $\{(\bx_{i}, \Tilde{\by}_{i})\}^{n_{1}+n_{2}}_{i = n_{1}+1}$.

3. We randomly choose $n_{3}$
samples  uniformly from the remaining to be the labeled set denoted  by
$\{(\bx_{i}, \Tilde{\by}_{i})\}^{n_{1}+n_{2}+n_{3}}_{i = n_{1}+n_{2}+1}$.

To evaluate the quality of an estimator of $\bB$, $\bB^{\ddag}$, since $\cF$ is unknown, we can first obtain an estimater of  $\cF$, $\hat{\cF}$, by minimizing $ \sum^{n_{1} + n_{2}}_{i = n_{1}  +1}\|\Tilde{\by}_{i} - \cF(\bB^{\ddag\top}\bx_{i})\|^{2}_{2}$ within a function class. Then, the quality of $\bB^{\ddag}$ can be evaluated by the prediction mean squared error (PMSE) $(1/n_{1}) \sum^{n_{1}}_{i=1}\|\Tilde{\by}_{i} - \hat{\cF}(\bB^{\ddag\top}\bx_{i})\|^{2}_{2}$. We choose $\cF$ belonging to the linear function class. More complex nonlinear function classes like those defined by neural networks can also be applied, but due to the limited samples, they do not perform well. Previous studies~\citep{scala2021phenotypic,kobak2021sparse} on the dataset show that linear functions are sufficient.       

We consider following variations of our first-order estimator:

1. semi-supervised estimator defined by equation~\eqref{eq:fss}, estimated on $\{\bx_{i}\}^{n}_{i = n_{1}+1} \cup \{  \Tilde{\by}_{i}\}^{n_{1}+n_{2}+n_{3}}_{i = n_{1}+n_{2}+1}$;

2. supervised estimator defined by equation~\eqref{eq:def-fe} with $\by_{i} = \Tilde{\by}_{i}$, estimated on $\{  (\bx_{i}, \Tilde{\by}_{i})\}^{n_{1}+n_{2}+n_{3}}_{i = n_{1}+n_{2}+1}$;

3. unsupervised estimator defined by equation~\eqref{eq:def-fe} with $\by_{i} = \bx_{i}$ and estimated on $\{ \bx_{i} \}^{n}_{i=n_{1}+1}$.

To make full use of the limited samples, our semi-supervised design is in a transductive learning way, i.e., $\{\bx_{i}\}^{n_{1}+n_{2}}_{i=n_{1}+1}$ is also accessible when estimate $\bB$. We also compare our estimators with PCA estimator. For choices of parameters and applications of competing methods, please refer to Section~\ref{sec:rda2} of the Supplement.

The performances of competing methods are depicted in Figure~\ref{fig:rda2}. 
We plot the medians of PMSE over 200 repetitions against embedding dimensions $h$ ranging from 4 to 8. The semi-supervised estimator consistently overwhelms the unsupervised estimator. It would also surpass the supervised method as the latent dimension increases. The two phenomenon jointly demonstrate that semi-supervised estimator could overwhelm the ones estimated solely from the feature or limited labels, respectively, in downstream tasks. Moreover, our unsupervised estimator consistently exceeds the PCA estimator due to leveraging nonlinear relationship among features.              
\begin{figure}[h]
\centering    \includegraphics[width=\textwidth]{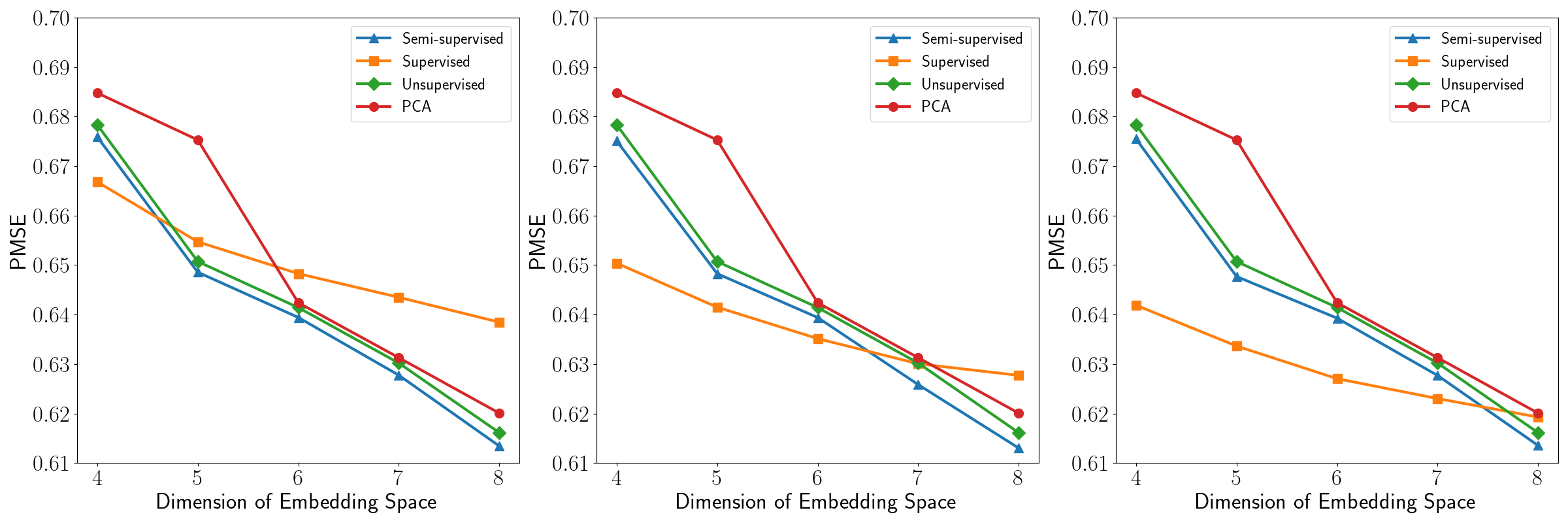}
\caption{PMSE on M1 Patch-seq Dataset, from left to right, $n_{2}$ equals to 50, 75 and 100, respectively. From left to right, the label size $n_{2}$ are 50, 75, 100, respectively.}\label{fig:rda2}
\end{figure}

\section{Discussion and Future Work}
In this paper, we introduced a general framework for learning the latent space within a nonlinear multiple response model setting. Our approach encompasses PCA as a special case when applied to Gaussian design. By eschewing gradient-descent-based methods for latent space learning, our approach is not only computationally more efficient but also offers strong theoretical guarantees. As a result, downstream tasks such as data compression and data transmission can greatly benefit from this framework.

It is crucial to recognize that Stein's lemmas necessitate knowledge of the score function. When the distribution family is known, we can estimate the unknown parameters, such as the mean and covariances, and employ a plug-in approach. However, in more general settings, accurately estimating the score function for both first-order and second-order methods remains a significant area of research. Additionally, this paper focuses on continuous inputs and outputs. To handle discrete labels, one possible solution is to use approximation methods to convert discrete labels into continuous ones. For instance, in the context of one-hot labels in a multi-class classification problem, the Gumbel-softmax trick can be employed for this conversion. Exploring how Stein's method can be adapted for general discrete scenarios will also be a major focus of our future investigations.

\bibliography{main.bib}
\bibliographystyle{ims}

\clearpage

\setcounter{page}{1}

\setcounter{section}{0}
\setcounter{equation}{0}

\setcounter{figure}{0}
\setcounter{table}{0}

\renewcommand{\theequation}{S\arabic{equation}}
\renewcommand{\thesection}{\Roman{section}}
\renewcommand*{\theHsection}{\Roman{section}}

\renewcommand\thefigure{S\arabic{figure}}
\renewcommand\thetable{S\arabic{table}}

\begin{center}
{\Large Supplementary Material for}

{\Large ``Nonlinear Multiple Response Regression and Learning of Latent Spaces"}

\vspace{.1in} {\large Ye Tian, Sanyou Wu and Long Feng}
\end{center}

\section{Some Further Discussions}
\subsection{Notations}\label{sec:app-notation}
We introduce some extra notations which would be used in the Supplement. For vector $\ba$, let $\| \ba 
 \|_{2} = \sqrt{\sum_{i}\ba^{2}_{i}} $ denote the vector Euclidean norm. For matrix $\bA$, let $\| \bA 
 \|_{2} = \sigma_{1}(\bA)$ denote the spectral norm, and  $\text{trace}(\bA)$ denote the trace of $\bA$. We let $\bI_{k}$ denote the identity matrix of dimension $k$.  We let $\text{vec}(\cdot)$ denote the vectorization operator, $\text{vec}^{-1}(\cdot)$ denote its inverse. For observations of random vectors $\bx_{1}, \ldots, \bx_{n} \overset{\text{i.i.d.}}{\sim} \bx$, let var($\bx$) denote its sample variance matrix.

\subsection{Discussion of the Validity of Assumptions in Section~\ref{sec:second-order}}\label{sec:app-as} 
Assumption~\ref{ass:second-order-subexp-1} is in line with Assumption~\ref{assum:1}, since if $\bs(\bx)$ is element-wise sub-Gaussian, then the corresponding $\bT(\bx)$ would be element-wise sub-exponential. 

Assumption~\ref{ass:second-order-subexp-2} is akin to Assumption~\ref{assum:2}, which posits that for all $j \in [q]$, $f_{j}(\bz)$ is essentially bounded.

Assumption~\ref{ass:second-order-subexp-3} also requires that $\epsilon$ is essentially bounded, which could achieved, for example, by assuming that $\epsilon$ follows truncated normal. Both Assumption~\ref{ass:second-order-subexp-2} and~\ref{ass:second-order-subexp-3} require essential boundedness. They are technical requirements to facilitate the analyses, and are not difficult to satisfy in practice when the sample size is finite. It is also demonstrated in our experiments in Section~\ref{sec:simulation} that even though $\epsilon$ or $f_{j}(\bz)$ are unbounded, second-order methods also exhibit good performances in many cases.  

Assumption~\ref{ass:second-order-subexp-4} naturally follows the fact that the second-order method can work only if $\bM_{2}$ is full-rank. While more link functions may cause fluctuations in $\lambda_{r}$, at the true value $\bB$, $\lambda_{r}(\bM_{2})$ still needs to be lower bounded by some absolute constant independent of $q$ for the second-order method to work. We assume that $\lambda_{r}(\bM_{2}) = \Omega(1/r)$ to ensure that the total signal level from eigenvalues can be at least $\Omega(1)$. Moreover, $ \|1/q \sum^{q}_{j=1} f_j(\bz) \|_{\infty} > C$ guarantees that $1/q \sum^{q}_{j=1} f_j(\bz)$ would not degenerate to a constant as more link functions are introduced.  

\subsection{Score Models}\label{sec:app-sm}

There are two main streams of methods of estimating the first-order score function: one, we call direct methods is to estimate $\bs(\bx)$ directly by minimizing certain measures of difference between $\bs(\bx)$ and its estimator $\hat{\bs}(\bx)$; the other, we call noisy methods, is to estimate some noisy version of $\bs(\bx)$ and hopefully, as noise levels get smaller, we can obtain a good estimator of $\bs(\bx)$.  
The direct methods can be divided into two categories: parametric methods and non-parametric kernel methods. We first introduce parametric methods. Let $\bs(\bx;\btheta)$ denote a parametric model of $\bs(\bx)$  parametrized by $\btheta$, the score macthing method \citep{hy2005} estimates $\btheta$ by minimizing the loss function 
\begin{subequations}
\begin{align}\label{eq:def-sm}
    \ell_{\text{sm}}(\btheta) = \frac{1}{2}\EE\left \{ \left \| \bs(\bx;\btheta) - \bs(\bx) \right\|^{2}_{2}\right \},
\end{align}
by applying integral by part, \citet{hy2005} show that minimizing \eqref{eq:def-sm} is equivalent to minimize 
\begin{align}\label{eq:def-smeq}
  J_{\text{sm}}(\btheta) = \EE\left[\text{tr}\left \{ \nabla_{\bx}\bs(\bx;\btheta) \right \} + \frac{1}{2} \left \|\bs(\bx;\btheta)  \right \|^{2}_{2}\right].  
\end{align}
\end{subequations}
Suppose we have observations $\bx_{1}, \ldots, \bx_{n} \overset{\text{i.i.d.}}{\sim} \PP(\bx)$, the score matching estimator $\hat{\btheta}_{\text{sm}}$ can be obtained by minimizing the empirical version of \eqref{eq:def-smeq}, i.e., 
\begin{align*}
\hat{\btheta}_{\text{sm}} = \text{argmin} \frac{1}{n} \sum_{i=1}^{n} \left[\text{tr}\left \{ \nabla_{\bx}\bs(\bx_{i};\btheta) \right\} + \frac{1}{2} \left \|\bs(\bx_{i};\btheta)  \right\|^{2}_{2}\right].
\end{align*}
Instead of measuring the squared $L_{2}$--norm between $\bs(\bx)$ and its estimator, sliced-score matching \citep{song2020} measures the difference by their projections on random directions from certain distributions. Specifically, they consider the loss function
\begin{subequations}
\begin{align}\label{eq:def-ssm}
    \ell_{\text{ssm}}\{\btheta, \PP(\bv)\} = \frac{1}{2}\EE_{\bv}\EE_{\bx}[ \bv^{\top}\{ \bs(\bx;\btheta) - \bv^{\top}\bs(\bx) \}^{2}],
\end{align}
by applying integral by part, \citet{song2019generative} show that minimizing \eqref{eq:def-ssm} is equivalent to minimize 
\begin{align}\label{eq:def-ssmeq}
  J_{\text{ssm}}\{\btheta, \PP(\bv)\} = \EE_{\bv} \EE_{\bx}\left[ \bv^{\top}\{\nabla_{\bx}\bs(\bx;\btheta)\}\bv  + \frac{1}{2}\{\bv^{\top}\bs(\bx;\btheta)  \}^{2}\right] 
\end{align}
\end{subequations}

Suppose for each $\bx_{i}$, we sample $\bv_{i1} \ldots, \bv_{im} \overset{\text{i.i.d.}}{\sim} \PP(\bv)$, the score matching estimator $\hat{\btheta}_{\text{ssm}}$ can be obtained by minimizing the empirical version of \eqref{eq:def-ssmeq}, i.e., 
\begin{align*}
\hat{\btheta}_{\text{ssm}} = \text{argmin} \frac{1}{nm} \sum_{i=1}^{n} \sum^{m}_{j=1} \left[ \bv_{ij}^{\top}\{\nabla_{\bx}\bs(\bx_{i};\btheta)\}\bv_{ij}  + \frac{1}{2}\{\bv_{ij}^{\top}\bs(\bx_{i};\btheta)  \}^{2}\right].
\end{align*}
In general, $\bs(\bx;\btheta)$ is chosen be some deep neural network $\cF_{\btheta}(\bx):\RR^{p} \rightarrow \RR^{p}$. 

Except for parametric models, there is another line of work considering the kernel-based score estimator, and many of these methods can be unified by the following framework\citep{zhou2020}. 

Let $\cX \subset \RR^{p}$ denote the support of $\bx$, $\cK$ denote a matrix-valued kernel $\cK: \cX \times \cX \rightarrow \RR^{p \times p}$ satisfying for any $\bx, \bx' \in \cX$, $\cK(\bx,\bx') = \cK(\bx',\bx)$ and for any $m \in N_{+}$, $\{\bx\}^{m}_{i=1} \subset \cX$ and $\{ \bc_{i} \}^{m}_{i=1}$, $\sum^{m}_{i,j=1}\bc^{\T}_{i}\cK(\bx_{i},\bx_{j})\bc_{j} \geq 0 $. We denote a vector-valued reproducing kernel
Hilbert space (RKHS) associated to $\cK$ by $\cH_{\cK}$. Then, the kernel estimator of $\bs(\bx)$ can be obtained by solving the regularized  vector-valued regression problem 
\begin{align}
\min_{\bs_{r}\in \cH_{\cK}} \frac{1}{n} \sum_{i=1}^{n} \|\bs(\bx_{i}) - \bs_{r}(\bx_{i})\|^{2}_{2} + R_{\lambda}(\bs_{r}),   
\end{align}
where $R_{\lambda}(\cdot)$ represents a regularizer belonging to the spectral class of regularization \citep{bauer2007} with tuning parameter $\lambda$ corresponding to the estimator of $\bs(\bx)$ defined as $g_{\lambda}(\hat{\cL}_{\cK})\hat{\cL}_{\cK}\bs(\bx)$, where $\hat{\cL}_{\cK} = \frac{1}{n} \sum^{n}_{i = 1}\cK(\bx_{i}, \cdot)$ and $g_{\lambda}(\hat{\cL}_{\cK})$ approximate the inverse of operator  $\hat{\cL}_{\cK}$.
Note that $\hat{\cL}_{\cK}:\cH_{\cK} \rightarrow \cH_{\cK}$ is positive semidefinite, suppose it has the eigen decomposition $\hat{\cL}_{\cK} = \sum_{i=1}^{np}\sigma_{i} \langle \cdot, U_{i} \rangle_{\cH_{\cK}} U_{i}$, then $g_{\lambda}(\hat{\cL}_{\cK})$ can be represented by $ g_{\lambda}(\hat{\cL}_{\cK}) = \sum_{i=1}^{np} g^{\sigma}_{\lambda}(\sigma_{i})\langle \cdot, U_{i} \rangle_{\cH_{\cK}} U_{i}$, where $g^{\sigma}_{\lambda}: \RR_{+} \rightarrow \RR$ is a scalar function. For example, the Tikhonov regularization corresponds to $g^{\sigma}_{\lambda}(\sigma) = 1/(\lambda + \sigma)$. In reality, $\bs(\bx)$ is unknown, so the solution is infeasible, but by integral by part, \citet{zhou2020} prove that $\cL_{\cK}\bs(\bx) = -\EE\{\text{div}_{\bx}\cK(\bx, \cdot)^{\T}\}$, therefore, empirically, we can estimate $\bs(\bx)$ by 
$\hat{\bs}_{r}(\bx) = - g_{\lambda}(\hat{\cL}_{\cK})\hat{\xi}$,  
where $\hat{\xi} = (1/n)\sum_{i=1}^{n}\text{div}_{\bx}\cK(\bx_{i}, \cdot)^{\T}$.
Many kernel score estimators like kernel exponential family estimator~\citep{strathmann2015gradient}, Stein gradient estimator~\citep{li2017gradient}, spectral Stein gradient estimator~\citep{shi2018spectral}, etc., can be included in the framework with different choices of the function $g^{\sigma}_{\lambda}$ and hypothesis space $\cH_{\cK}$. For more details please refer to~\citet{zhou2020} and references therein.  

All those direct methods mentioned above use integral by parts to avoid the unknown function $\bs(\bx)$. 

In contrast, noisy methods circumvent the unknown $\bs(\bx)$ by estimating a noisy version of $\bs(\bx)$. Specifically, denoising score matching \citep{vincent2011} considers the estimator of $\btheta$ 
minimizing the following loss function 
\begin{subequations}
\begin{align}\label{eq:def-dsm-1}
\EE_{(\bx, \Tilde{\bx})} \left \| \bs(\Tilde{\bx};\btheta) - \frac{\partial \ln{P_{\sigma}(\Tilde{\bx}|\bx)}}{\partial\Tilde{\bx}} \right  \|^{2}_{2},   
\end{align}
where $P_{\sigma}(\Tilde{\bx}|\bx)$ is some smooth density estimator with bandwidth $\sigma$, \citet{vincent2011} prove that under mild conditions minimizing \eqref{eq:def-dsm-1} is equivalent to minimizing 
\begin{align*}
\EE_{\Tilde{\bx}} \left \| \bs(\Tilde{\bx};\btheta) - \frac{\partial \ln{P_{\sigma}(\Tilde{\bx})}}{\partial\Tilde{\bx}} \right  \|^{2}_{2},   
\end{align*}
where $P_{\sigma}(\Tilde{\bx}) = (1/n)\sum_{i=1}^{n} \ln{P_{\sigma}(\Tilde{\bx}|\bx_{i})}$. Hopefully, as $\sigma$ gets close to $0$, the noisy model $\bs(\bx;\hat{\btheta}_{\text{dsm}})$ would be a good approximation of $\bs(\bx)$. Empirically, suppose for each $\bx_{i}$, we generate  $\Tilde{\bx}_{i,1}, \ldots \Tilde{\bx}_{i,m} \overset{\text{i.i.d.}}{\sim} P_{\sigma}(\Tilde{\bx}|\bx_{i})$, then we can get the estimator $\hat{\btheta}_{\text{dsm}}$ by minimizing the following empirical loss function:
\begin{align*}
\frac{1}{nm}\sum_{i=1}^{n}\sum_{j=1}^{m} \left \| \bs(\Tilde{\bx}_{i,j};\btheta) - \frac{\partial \ln{P_{\sigma}(\Tilde{\bx}_{i,j}|\bx_{i})}}{\partial\Tilde{\bx}_{i,j}} \right  \|^{2}_{2}.
\end{align*}
\end{subequations}
$\bs(\bx;\btheta)$ is also generally chosen to be a deep neural network $\cF_{\btheta}(\bx):\RR^{p} \rightarrow \RR^{p}$. Some modifications and extensions were proposed. For example, \citet{song2019generative} proposed to train the score model using a geometric series of levels of $\sigma$.

Compared with the rapid development of first-order score models due to the prevalence of diffusion models, there is much limited work on second-order score models. For kernel methods, there's no clear function space corresponding to second order gradients; for neural networks, directly taking gradient of first-order score models with respect to predictors not only leads to error accumulation but also in general, cannot generate symmetric output. Moreover, building a neural network for the second-order score also faces up constrains like symmetry, which makes the problem difficult. To the best of our knowledge,~\citet{ meng2021estimating} is the only work that propose a deep neural network model for the second-order score function. However, their method needs to estimate a first-order score model at the same time, which leads to error accumulation.

\subsection{Reduced Rank Estimator, Neural Network Estimator and Ours}\label{sec:app-rno}

The reduced rank regression addresses the following multiple-response linear regression model:\begin{align}\label{eq:def-lin-fac-reg}
    \by = \bC^{\top}\bx + \varepsilon,
\end{align}
where $\bC \in \mathbb{R}^{p\times q}$, and $\text{rank}(\bC) \leq r$ for some $r \leq \min(p,q)$. 
Let $\bX$ denote the data matrix $(\bx_{1}, \ldots, \bx_{n})^{\top}$ and $\bY$ denote the response matrix $(\by_{1}, \ldots, \by_{n})^{\top}$, reduced rank regression estimates the coefficient matrix $\bC$ by solving the constrained optimization problem: $
    \hat{\bC} = \text{argmin}_{\text{rank}(\bC) \leq r} \| \bY  - \bX \bC   \|^{2}_{F}$.

Under low-dimensional setting, i.e., $\text{rank}(\bC)$ is fixed, it is well known that if $r$ is given, $\hat{\bC}$ has the closed form \citep{mukherjee2011}: $\hat{\bC} = \hat{\bC}_{ols}\bV_{r}\bV_{r}^{\top}$. $\hat{\bC}_{ols}$ is the ordinary least squares estimator, i.e., 
$\hat{\bC}_{ols} = \text{argmin} \| \bY  = \bX \bC   \|^{2}_{F}$. $\bV_{r} = \text{SVD}_{l,r}\{(\bX \hat{\bC}_{ols})^{\T}\}$, i.e., the matrix consisting of the first $r$ leading right singular vectors of $\bX \hat{\bC}_{ols}$.   

Notably, in our model \eqref{model:1}, if we additionally require that $r < q$, and$ f_{j}(\bv) = \ba^{\top}_{j}\bv$ for certain $\ba_{j} \in \RR^{r}, j \in [q]$, it reduces to model \eqref{eq:def-lin-fac-reg} with $\bC = \bB\bA$, where $\bA = (\ba_{1}, \ldots, \ba_{q})$. Therefore,  model \eqref{eq:def-lin-fac-reg} is a special case of our model \eqref{model:1}, and $\hat{\bC}$ can be used for estimating $\bB$. The reduced rank regression estimator of $\bB$ is defined as $\widehat{\bB}_{R} = \text{SVD}_{l,r}\left( \hat{\bC}\right )$.   
\begin{remark}\label{rem:f-r}
Under the assumption that $\bx \sim \mathcal{N}(0, \bSigma_{\mathcal{N}})$, and $\bSigma_{\mathcal{N}}$ is non-degenerate, by Lemma~\ref{lem:gaussian-score}, $\bs(\bx) = \bSigma^{-1}_{\mathcal{N}}\bx$, if we let $\hat{\bs}(\bx) = (1/n)(\bX^{\top}\bX)^{-1}\bx$, then $\widehat{\bB} = \text{SVD}_{l,r}\left( \hat{\bC}_{ols}\right )$. The difference between $\widehat{\bB}_{R}$ and $\widehat{\bB}$ lies in the projection matrix $\bV_{r}\bV_{r}^{\top}$, which can be seen as the benefit of $\widehat{\bB}_{R}$ taking advantage of the extra linear information of link functions in model \eqref{eq:def-lin-fac-reg}. Results of simulations in Section~\ref{sec:simulation} demonstrate that performances of $\widehat{\bB}$ and $\widehat{\bB}_{R}$ are almost the same when $\bx \sim \mathcal{N}(0, \bSigma_{\mathcal{N}})$, the difference above is almost negligible. 
\end{remark}   

The multi-index model has a close relationship with NNs, and many works try to show that NNs can be used to estimate MIM and construct the representation space in the low-dimensional setting~\citep{Benedikt18, damian2022neural, mousavi2022neural}. Similarly, for our multi-response extension, the matrix 
$\bB$ can also be estimated by NNs. Let $\cF_{\btheta}(\cdot) : \mathbb{R}^{r} \rightarrow \mathbb{R}^{q}$ be a neural network parametrized by $\btheta$, the neural network estimator $\widehat{\bB}_{N}$ can be obtained by 
\begin{align}\label{eq:def_esbn}
    (\hat{\btheta}_{N}, \widehat{\bB}_{N}) = \text{argmin}_{(\btheta, \bB)} \frac{1}{n}\sum_{i=1}^{n}\| \by_{i} - \cF_{\btheta}(\bB^{\top}\bx_{i}) \|^{2}_{2}.
\end{align}
We solve the optimization problem by mini-batch gradient descent. For details of the neural network structure and training procedures, please see Section~\ref{sec:app-nn} of the Supplement.

To apply our method, by densities of $\bx$, first-order and second-order score functions can be derived in closed forms. Please refer to Section~\ref{sec:app-sf} of Supplement for details. Then, for $\bx \sim \mathcal{N}(0, \bSigma_{\mathcal{N}})$, parameters are estimated by maximum likelihood estimation. For $\bx \sim t_{\nu}(0, \bSigma_{t})$ and $\bx \sim \cH_{\chi, \psi}(0, \bSigma_{\cH})$, parameters are estimated by a multi-cycle, expectation, conditional estimation algorithm~\citep{breymann2013ghyp}. Then we use the plug-in estimators $\hat{\bs}(\cdot)$ and $\hat{\bT}(\cdot)$ to calculate $\widehat{\bB}$ and $\Tilde{\bB}$ defined by equation~\eqref{eq:def-fe} and equation~\eqref{eq:def-se}, respectively. 

\subsection{Extended Experiments and Further Discussion}\label{sec:app-ee}
In this subsection, we provide results of experiments on additional choices of $p$ in $\{ 50, 80, 100 \}$, and $p = 30$ in the main paper, which are depicted in Figures~\ref{fig:app-0} to~\ref{fig:app-3}. Except results similar to those discussed in Section~\ref{sec:exp-dis}, comparing Figures~\ref{fig:app-0} to~\ref{fig:app-3}, we see that, while the second-order method tend to overwhelm other methods as the sample size increases, the minimum sample size that it needs to overtake others increases as $p$ increases, which coincides with our analysis of the higher order dependency of its convergence rate on $p$ in Section~\ref{sec:second-order}. 

\begin{figure}[h]
\begin{subfigure}{\textwidth}
\centering    \includegraphics[width=\textwidth,height=0.3\linewidth]{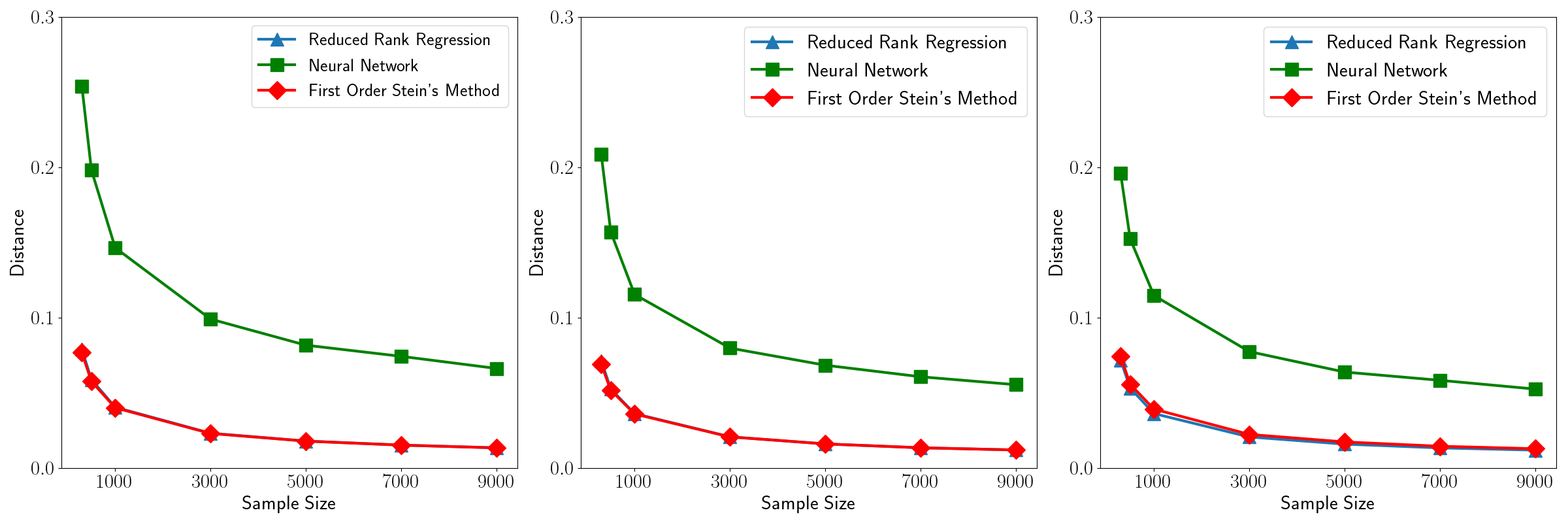}
\caption{linear link functions }
\end{subfigure}
\begin{subfigure}{\textwidth}
\centering    \includegraphics[width=\textwidth,height=0.3\linewidth]{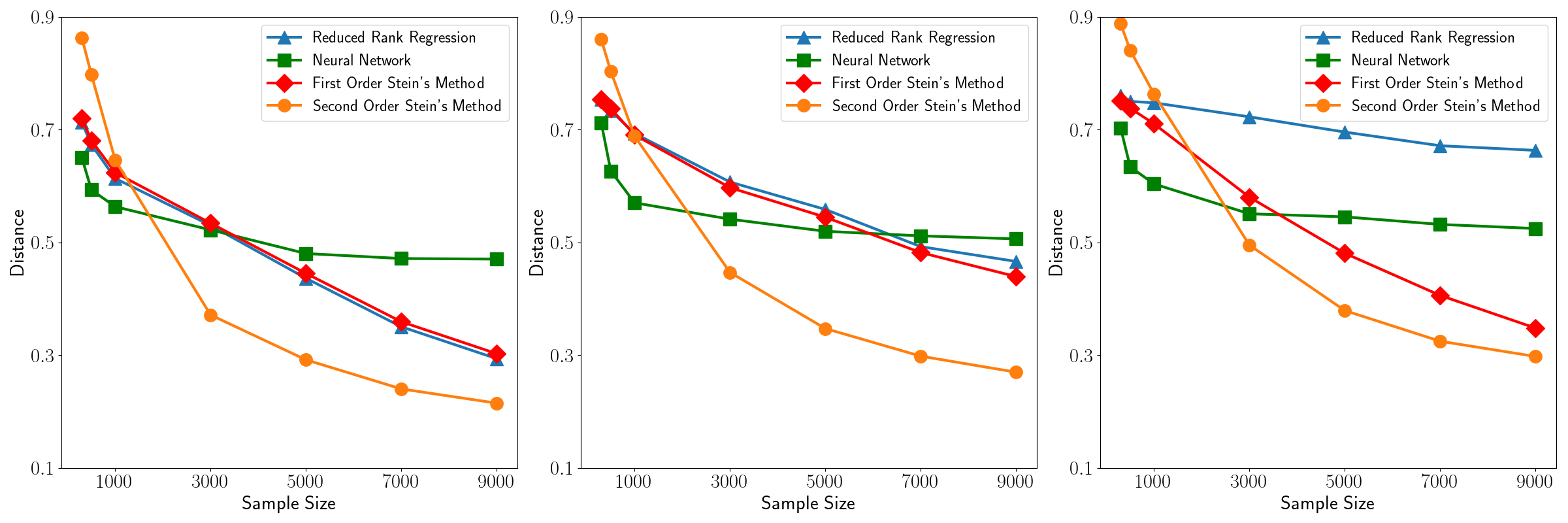}
\caption{the first mechanism of generating nonlinear link functions}
\end{subfigure}
\begin{subfigure}{\textwidth}
\centering    \includegraphics[width=\textwidth,height=0.3\linewidth]{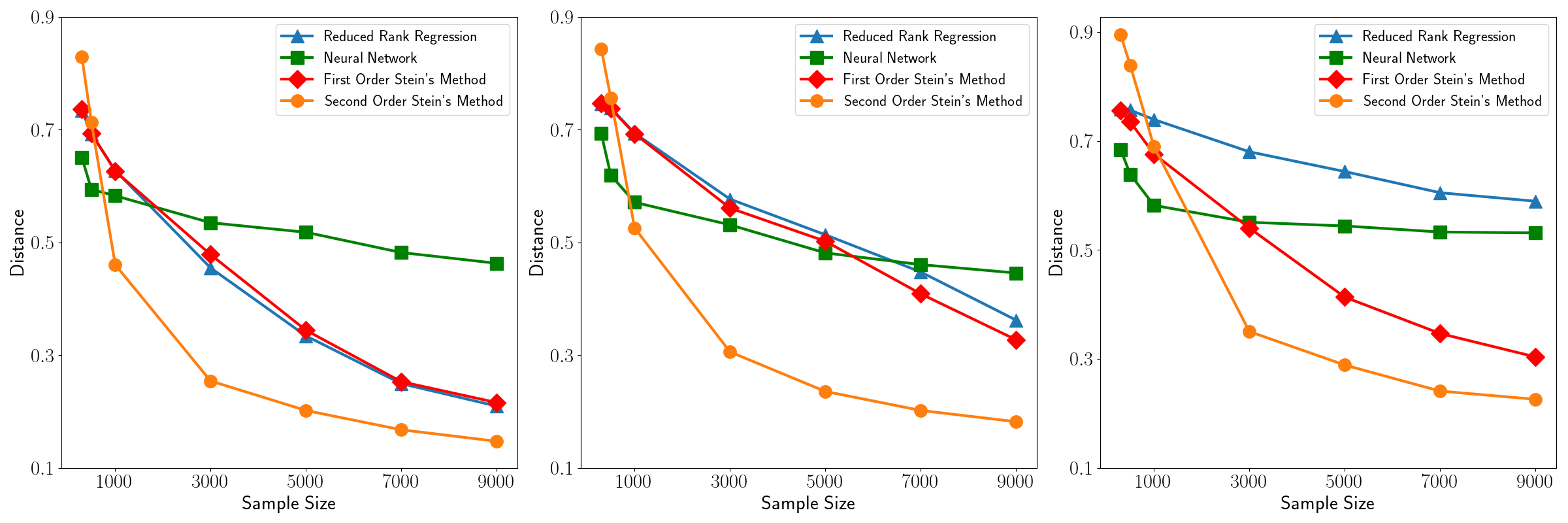}
\caption{the second mechanism of generating nonlinear link functions}
\end{subfigure}
\caption{The figure demonstrates finite sample performances of competing methods when $p = 30$. From left to right, $\bx \sim \mathcal{N}(0, \bSigma_{\mathcal{N}})$, $ \mathcal{H}_{\chi,\psi}(0, \bSigma_{\mathcal{H}})$ and $t_{\nu}(0, \bSigma_{t})$, respectively.}\label{fig:app-0}
\end{figure}

\begin{figure}[p]
\begin{subfigure}{\textwidth}
\centering    \includegraphics[width=\textwidth,height=0.3\linewidth]{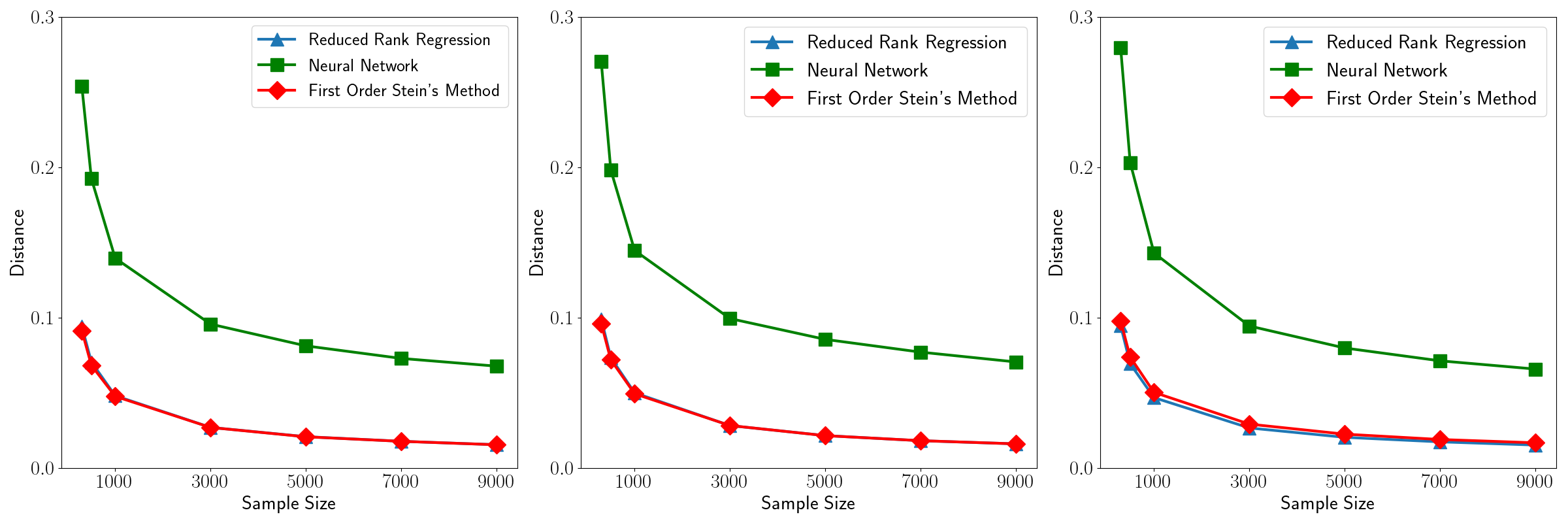}
\caption{linear link functions}
\end{subfigure}
\begin{subfigure}{\textwidth}
\centering    \includegraphics[width=\textwidth,height=0.3\linewidth]{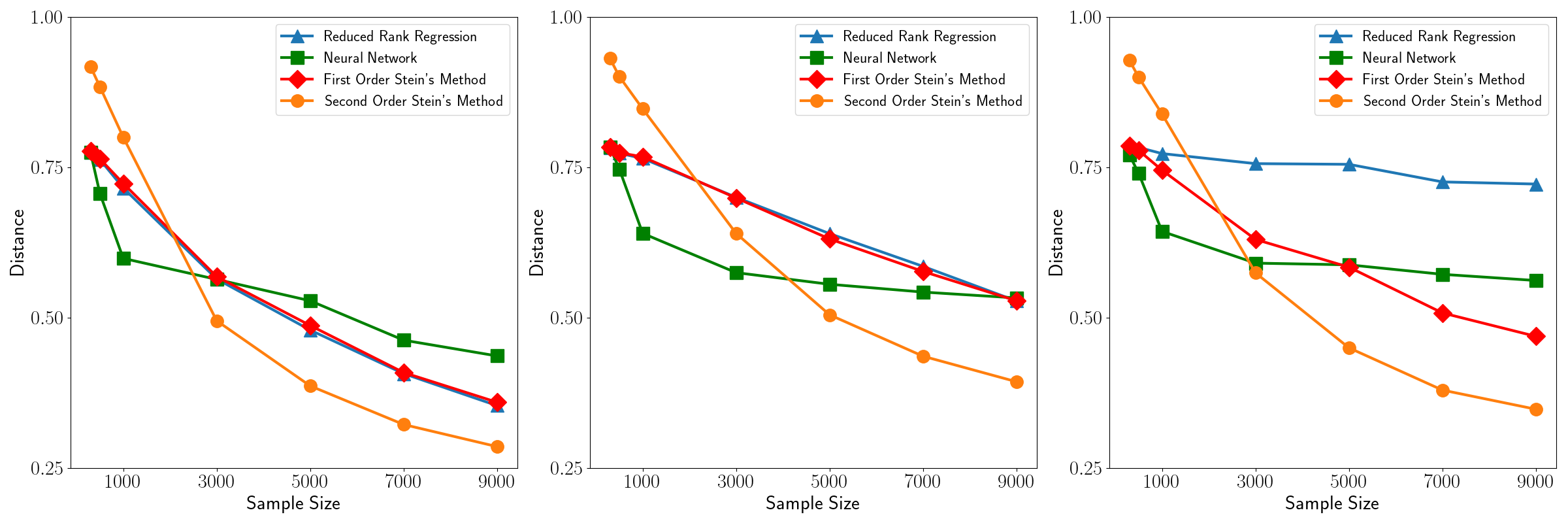}
\caption{the first mechanism of generating nonlinear link functions}
\end{subfigure}
\begin{subfigure}{\textwidth}
\centering    \includegraphics[width=\textwidth,height=0.3\linewidth]{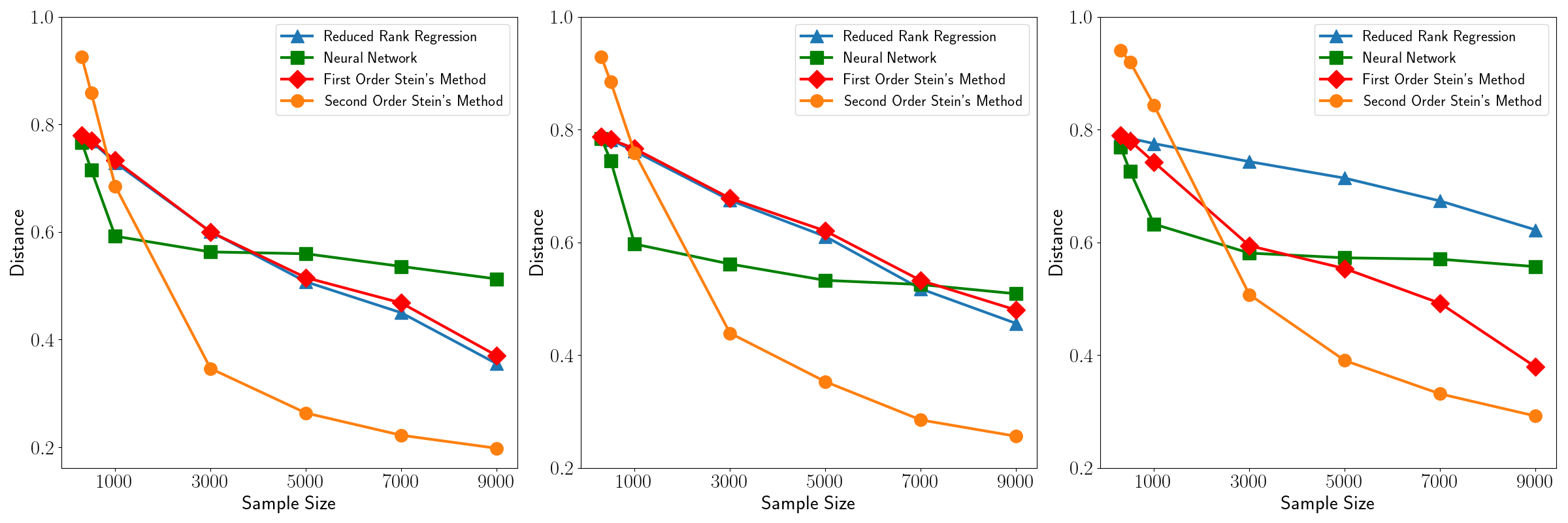}
\caption{the second mechanism of generating nonlinear link functions}
\end{subfigure}
\caption{The figure demonstrates finite sample performances of competing methods when $p = 50$, from left to right, $\bx \sim \mathcal{N}(0, \bSigma_{\mathcal{N}})$, $ \mathcal{H}_{\chi,\psi}(0, \bSigma_{\mathcal{H}})$ and $t_{\nu}(0, \bSigma_{t})$, respectively. }\label{fig:app-1}
\end{figure}

\begin{figure}[p]
\begin{subfigure}{\textwidth}
\centering    \includegraphics[width=\textwidth,height=0.3\linewidth]{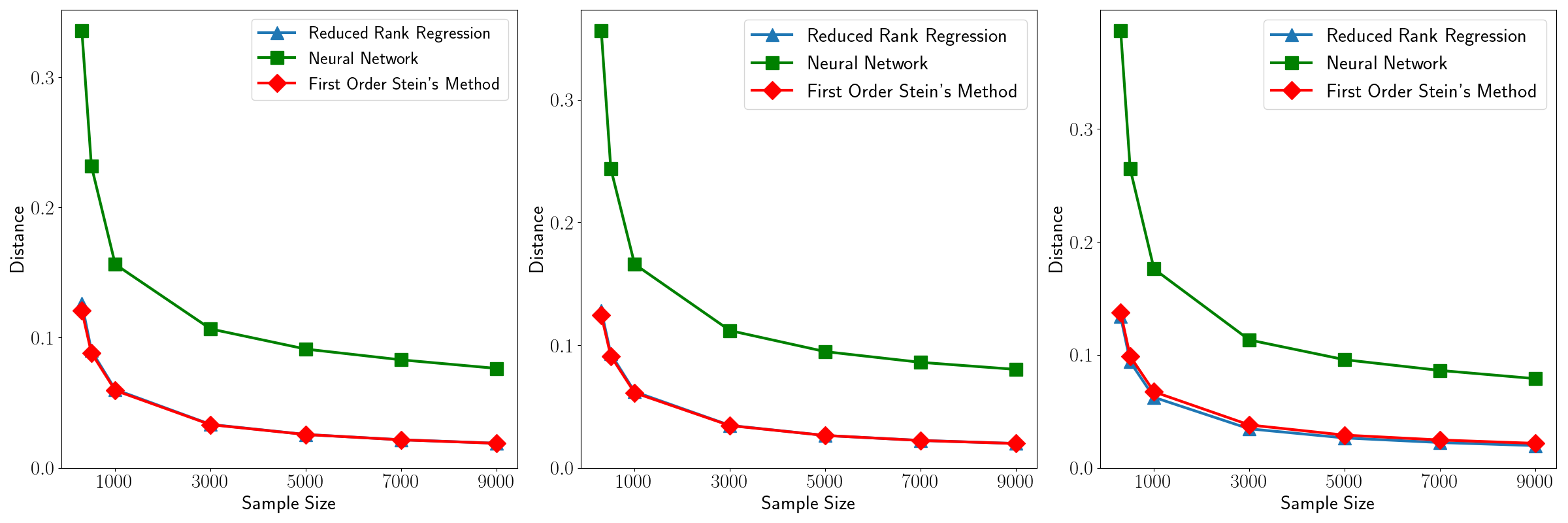}
\caption{linear link functions}
\end{subfigure}
\begin{subfigure}{\textwidth}
\centering    \includegraphics[width=\textwidth,height=0.3\linewidth]{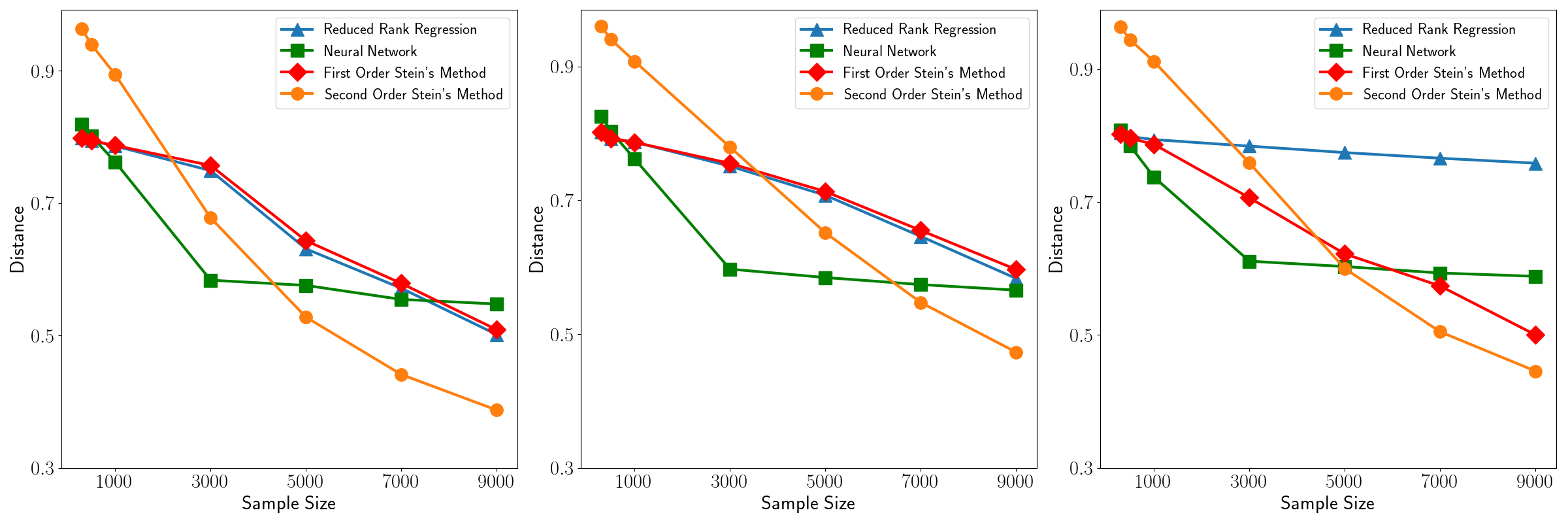}
\caption{the first mechanism of generating nonlinear link functions}
\end{subfigure}
\begin{subfigure}{\textwidth}
\centering    \includegraphics[width=\textwidth,height=0.3\linewidth]{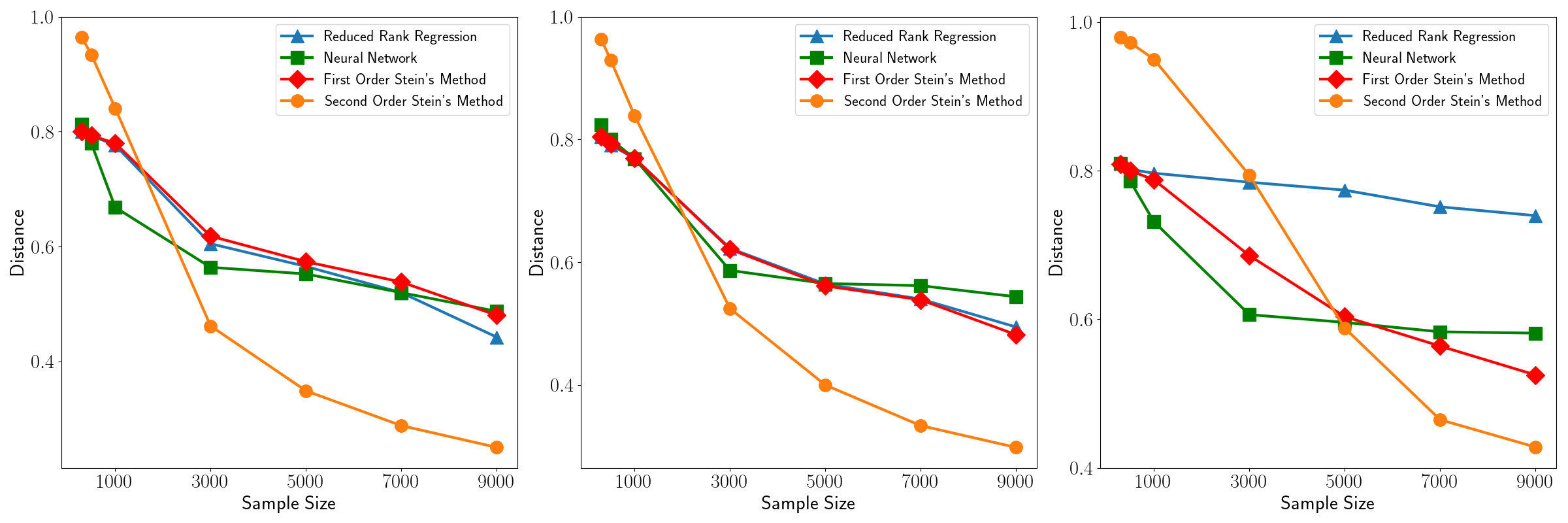}
\caption{the second mechanism of generating nonlinear link functions}
\end{subfigure}
\caption{The figure demonstrates finite sample performances of competing methods when $p = 80$, from left to right, $\bx \sim \mathcal{N}(0, \bSigma_{\mathcal{N}})$, $ \mathcal{H}_{\chi,\psi}(0, \bSigma_{\mathcal{H}})$ and $t_{\nu}(0, \bSigma_{t})$, respectively.}\label{fig:app-2}
\end{figure}

\begin{figure}[p]
\begin{subfigure}{\textwidth}
\centering    \includegraphics[width=\textwidth,height=0.3\linewidth]{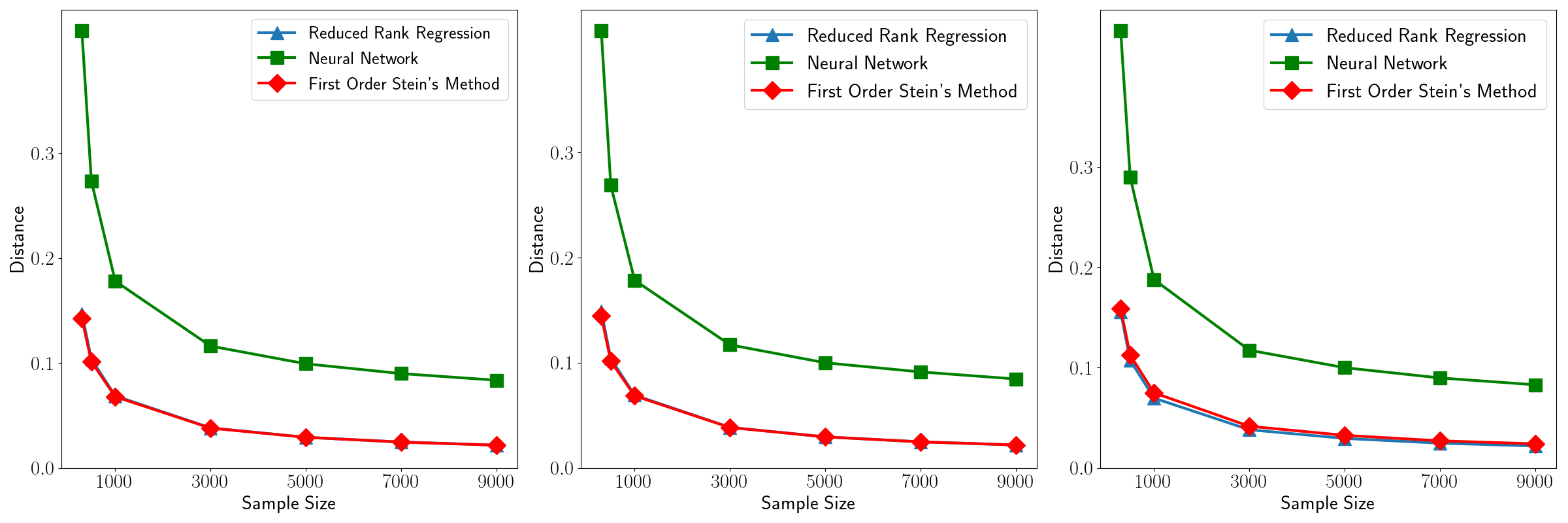}
\caption{linear link functions}
\end{subfigure}
\begin{subfigure}{\textwidth}
\centering    \includegraphics[width=\textwidth,height=0.3\linewidth]{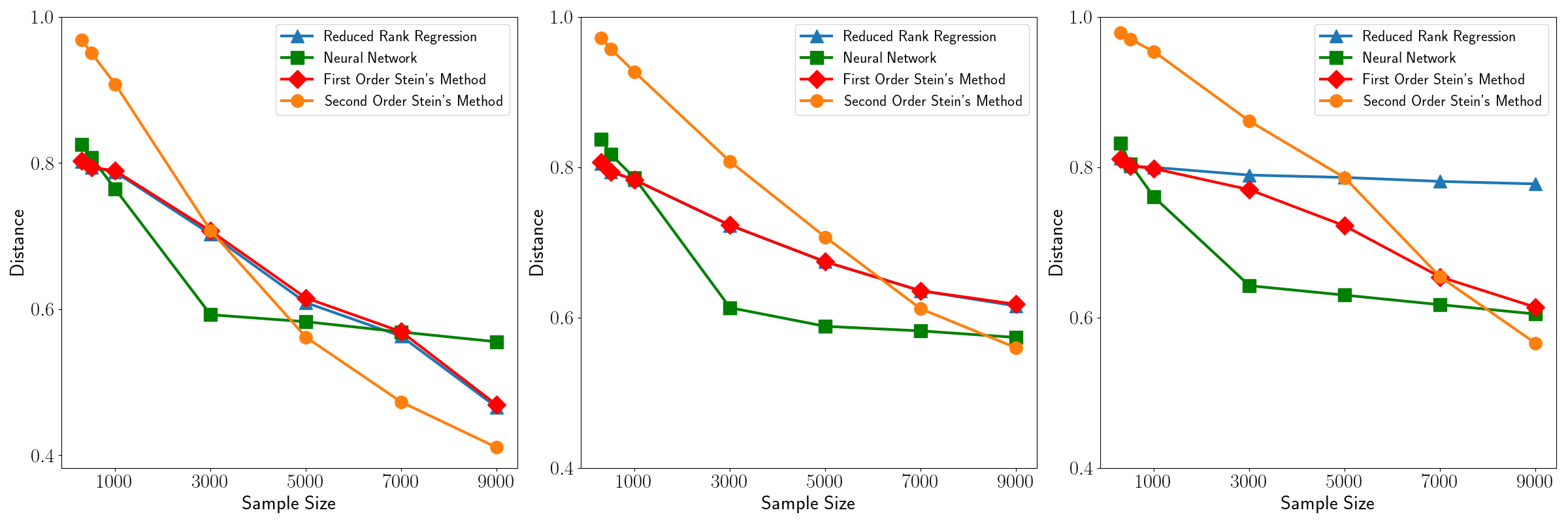}
\caption{the first mechanism of generating nonlinear link functions}
\end{subfigure}
\begin{subfigure}{\textwidth}
\centering    \includegraphics[width=\textwidth,height=0.3\linewidth]{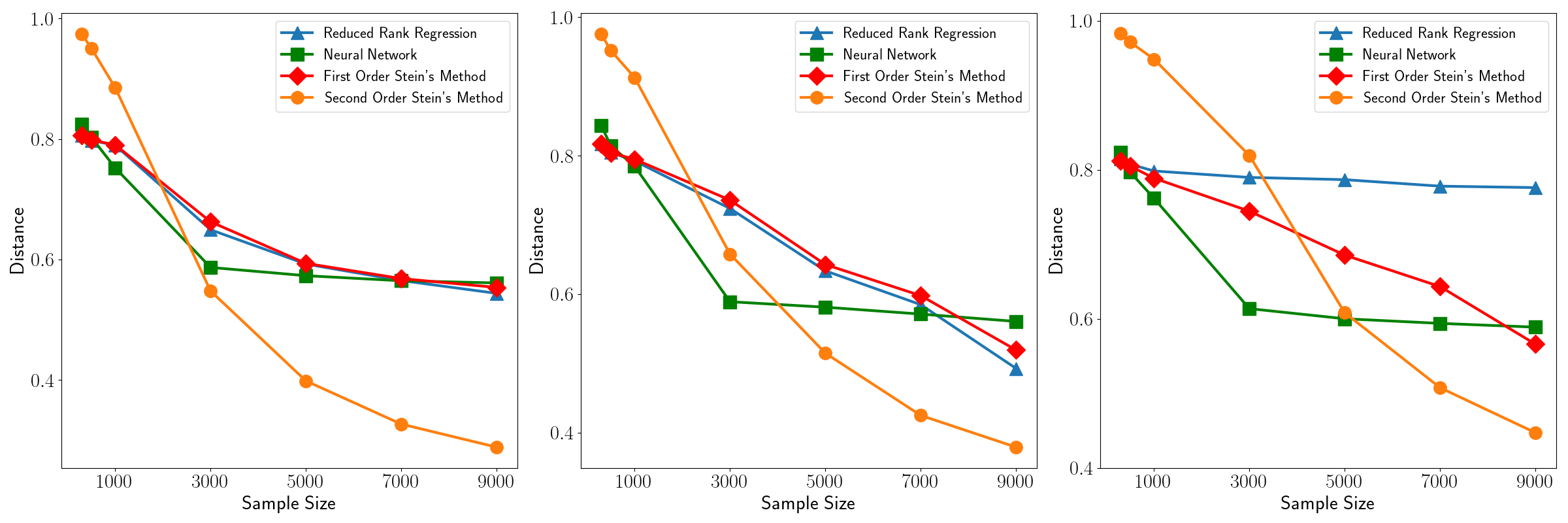}
\caption{the second mechanism of generating nonlinear link functions}
\end{subfigure}
\caption{The figure demonstrates finite sample performances of competing methods when $p = 100$, from left to right, $\bx \sim \mathcal{N}(0, \bSigma_{\mathcal{N}})$, $ \mathcal{H}_{\chi,\psi}(0, \bSigma_{\mathcal{H}})$ and $t_{\nu}(0, \bSigma_{t})$, respectively. }\label{fig:app-3}
\end{figure}

\subsection{Measures of Qualities of Estimates of Latent Spaces in Section~\ref{sec:rda} }\label{sec:app-measures}
Suppose image data $\bX \in \RR^{c \times c} \sim \PP(\bX)$, we have an encoder $\cE(\bZ):\mathbb{R}^{c \times c} \rightarrow \mathbb{R}^{h}$, and a decoder $\cD(\bz):\mathbb{R}^{h} \rightarrow \mathbb{R}^{c \times c}$, where $h$ denotes the embedding dimension. 
We first consider two metrics measuring (dis--)similarities between the original image $\bX$ and the recovered image $\hat{\bX} = \cD \circ \cE(\bX) $.  
\begin{itemize}
    \item [1. ] We consider the normalized root squared error (NRSE), i.e., $\EE_{\bX}(\|\hat{\bX} - \bX\|_{F}/\| \bX \|_{F})$
    \item [2. ] We also consider the structural similarity index measure (SSIM) between $\hat{\bX}$ and $\bX$, i.e., $\EE_{\bX}\{\text{SSIM}(\hat{\bX}, \bX)\}$. SSIM calculates the similarity score between two images by comparing their luminance, contrast, and structure; it ranges from -1 to 1 and the larger the SSIM is, the more similar the images are. For the definition of SSIM, please refer to Section \ref{sec-app-ssim} of the Supplement for details. 

    \item [3. ] Specifically, we consider the classification accuracy defined by: 
    \begin{align*} 
        \EE[ \II \{\text{argmax} \circ \; 
 \cC_{\cE} \circ \cE(\bX) = y\}],
    \end{align*}
where $\II(\cdot)$ is the indicator function and $ \cC_{\cE}(\cdot):\RR^{h} \rightarrow \RR^{10}$ is a multinomial logistic regression model trained on the set $[\{ \cE(\bX_{i}), y_{i} \}]^{n}_{i=1}$. For details of the model $\cC_{\cE}(\cdot)$, please refer to Section \ref{sec:app-cl} of the Supplement.
\end{itemize}

\subsection{Application Details of Competing Methods in Section~\ref{sec:rda}}\label{sec:app-app-rda}
Let $\widehat{\bB}_{\text{PCA}} = \text{Eigen}_{h}[\text{var}\{\text{vec}(\bX)\}]$.
For PCA method, we let the encoder to be $\cE_{\text{PCA}}(\bX) = \widehat{\bB}_{\text{PCA}}^{\top} \text{vec}(\bX)$ and the decoder to be $\cD_{\text{PCA}}(\bx) = \text{vec}^{-1}\{\widehat{\bB}_{\text{PCA}} \bx\}$.\footnote{Since pixel values of original images are in [0,1], we use the clip function $\text{clip}(x, \alpha, \beta) =\min\{\max(x, \alpha), \beta\}$ to scale the output of the decoder. For fair comparison, we handle the range problem in the same way for all methods, instead of constraining the range of output in the structure of the decoder.}
We employ fully connected neural networks for AE, encoder  $\cE_{\text{AE}, \btheta_{e}}(\bX)$ and decoder $\cD_{\text{AE}, \btheta_{d}}(\bx)$, are parametrized by $\btheta_{e}$ and $\btheta_{d}$ respectively, which can be estimated by minimizing the reconstruction squared error loss. Empirically,  $(\hat{\btheta}_{e}, \hat{\btheta}_{d})$ is the minimizer of the following loss function: 
\begin{align}\label{eq:loss-ae}
 \ell(\btheta_{e}, \btheta_{d}) = \frac{1}{n} \sum_{i=1}^{n}\left \{\|\cD_{\text{AE}, \btheta_{d}} \circ \cE_{\text{AE}, \btheta_{e}}(\bX_{i})  - \bX_{i}\|^{2}_{F}\right \}. 
\end{align}
We also use batch--training methods to solve the optimization problem. For details of the structure of the AE and training procedures, please refer to Section \ref{sec:app-st} of the Supplement. 

For our first-order estimator, we use the pre-trained score model for MNIST from \citet{song2019generative}  as our first-order score estimator denoted as $\hat{\bs}(\cdot)$\footnote{The pre-trained score model from \citet{song2019generative} takes the matrix-valued image data and output the score with the same shape as the input, to make it coincide with our setup, we flat the score into a vector.}. Then the plug-in estimator $\check{\bB}$ equals $\text{SVD}_{l,h}[(1/n)\sum^{n}_{i=1}\text{vec}\{\hat{\bs}(\bX_{i})\}\{\text{vec}(\bX_{i})\}^{\T}]$, and our first-order encoder is defined as $\cE_{S_{f}}(\bX) = \check{\bB}^{\top} \text{vec}(\bX)$. 
Our first-order decoder has the same structure as that of the AE, parametrized by $\btheta_{S_{f}}$ and denoted as $\cD_{S_{f}, \btheta_{S_{f}}}(\cdot)$; $\btheta_{S_{f}}$
can also be estimated by minimizing the following empirical mean reconstruction squared error losses on the training set:
\begin{align}\label{eq:def-sdoc}
 \ell(\btheta_{S_{f}}) = \frac{1}{n}\sum^{n}_{i=1}\{\|\cD_{S_{f}, \btheta_{S_{f}}}  \circ \cE_{S_{f}}(\bX_{i})  - \bX_{i}\|^{2}_{F}\} 
\end{align}    
The decoder is trained in the way similar to that of the AE, please see Section \ref{sec:app-st} of the Supplement for details. 

For our second-order estimator, since there is a lack of trustworthy second-order score models as discussed in Section~\ref{sec:app-sm}, we assume pixel values in an vectorized image data follow a multivariate normal distribution and use the estimator of second-order stein's score of multivariate normal distribution introduced in Section~\ref{sec:app-sf} as the second-order score estimator $\hat{\bT}(\cdot)$. We still use $\Tilde{\bB}$ to denote the second-order plug-in estimator, which equals  $\text{SVD}_{l,h}[(1/n/c^{2}\sum^{n}_{i=1}\sum^{c^{2}}_{j = 1}\text{vec}(\bX_{i})_{j}\hat{\bT}\{\text{vec}(\bX_{i})\}]$, and  the second-order encoder is defined as $\cE_{S_{s}}(\bX) = \Tilde{\bB}^{\top} \text{vec}(\bX)$. The second order decoder $\cD_{S_{s}, \btheta_{S_{s}}}(\bX)$ has the same structure as the first-order decoder and is trained in the same way.

\section{Details of Simulations}
\subsection{Parameters and Link Functions }\label{sec:app-pln}
Throughout the experiment, we let $q = 20$; $p \in \{30,50,80,100\}$; $n \in \{300, 500, 1000, 3000, 5000, 7000, 9000\}$; $r = 3$;
$\mu_{o}=0$,
$\sigma_{o}=1$;
$\sigma_{\epsilon}=0.5$; $\nu = 10$; $\chi=2p+1$, $\psi = p$.\\
In the first case, we let $\ba_{1}, \ldots, \ba_{q}\overset{\text{i.i.d.}}{\sim}\mathcal{N}(0, 0.5^{2}\bI_{r})$; in the second and third case, we let $\ba_{1}, \ldots, \ba_{q}\overset{\text{i.i.d.}}{\sim} \ba$ and elements of $\ba$ are i.i.d. samples from $|z| + 3$, where $z \sim \mathcal{N}(0,1)$. 
 
We generate $\bSigma_{\mathcal{N}}$, $\bSigma_{t}$ and $\bSigma_{\cH}$ in the following three steps:
\begin{enumerate}
    \item draw a random orthogonal matrix $\bQ$ from the $O(p)$ Haar distribution;
    \item generate a diagonal matrix $\bLambda_{(\cdot)}$, whose diagonal elements $\bLambda_{i,i} \overset{\text{i.i.d.}}{\sim} |z| + b_{(\cdot)}$, where $z \sim \mathcal{N}(0,\sigma^{2}_{(\cdot)})$;
    \item let $\bSigma_{(\cdot)} = \bQ \bLambda_{(\cdot)} \bQ^{\top}$.
\end{enumerate}
We choose $b_{\mathcal{N}} = b_{t} = b_{\cH} = 1$, $\sigma_{\mathcal{N}} = \sigma_{t} = \sigma_{\cH} = 1$.

We use the following elementary nonlinear functions:
\begin{enumerate}
    \item $m_{1} \coloneqq \sin(x-1)$
    \item $m_{2} \coloneqq \text{cosh}(x-1)$
    \item $m_{3} \coloneqq \text{cos}(x-1)$
    \item $m_{4} \coloneqq \text{tanh}(x-1)$
    \item $m_{5} \coloneqq \text{arctan}(x-1)$    
    \item $m_{6} \coloneqq (x - 1)^{3}$
    \item $m_{7} \coloneqq (x - 1)^{5}$
    \item $m_{8} \coloneqq 1/\{1 + \exp(-x)\}$
    \item $m_{9} \coloneqq\sqrt{(x-1)^{2} + 1}$
    \item $m_{10} \coloneqq \exp(x)$
\end{enumerate}

\subsection{Multivariate Hyperbolic Distribution}\label{sec:app-mhyp}
The random vector $\bx$ is said to have a multivariate  generalized hyperbolic distribution if
\begin{align*}
\bx \overset{\text{d}}{=} \bu + w \bgamma + \sqrt{w} \bA\bz,
\end{align*}
where $\bz \sim \cN(0, \bI_{o})$, $\bA \in \RR^{p \times o}$, $\bmu, \bgamma \in \RR^{p}$, $w \geq 0$ is a random variable, independent of $\bz$ and has a Generalized Inverse Gaussian distribution, written as $\textbf{GIG}(\lambda, \chi, \psi)$. The density is 
\begin{align}
    f(\bx) = \frac{(\sqrt{\psi/\chi})^{\lambda}(\psi+\bgamma^{\T}\bSigma^{-1}\bgamma)^{p/2-\lambda}}{(2\pi)^{p/2}|\bSigma|^{1/2}K_{\lambda}(\sqrt{\chi\psi})} \times \frac{K_{\lambda-p/2}(\sqrt{\{ \chi + Q(\bx) \}(\psi + \bgamma^{\T}\bSigma^{-1}\bgamma)})\exp\{ (\bc - \bmu)^{\T}\bSigma^{-1} \bgamma\}}{\{ \sqrt{\{ \chi + Q(\bx) \}(\psi + \bgamma^{\T}\bSigma^{-1}\bgamma)} \}^{p/2 - \lambda}},
\end{align}
where $Q(\bx) = (\bx - \bmu)^{\T}\bSigma^{-1}(\bx - \bmu)$, $K_{\lambda}(\cdot)$ is  the modified Bessel function of the third kind. In the experiment, we choose $\lambda = \frac{p+1}{2}$, and fix 
$\bgamma = 0$, $\bmu=0$, which is named multivariate hyperbolic 
distribution.

\subsection{Structures and Training Details of Neural Networks}\label{sec:app-nn}

For the neural network, we let $\cF_{\btheta}(\cdot) = \bO^{\top} \phi\{\bA^{\top}\phi(\cdot)\}$, where $\bA \in \mathbb{R}^{r \times h}$, $\bO \in \mathbb{R}^{h \times q}$, $\phi(\cdot)$ is the point-wise ReLU function and we let $h = 32$.
We implement neural networks in Pytorch. For neural network estimator of $\bB$ discussed in Section \ref{sec:app-rno}, we minimize the loss function 
\begin{align*}
\ell(\btheta, \bB) = \frac{1}{n}\sum_{i=1}^{n}\| \by_{i} - \cF_{\btheta}(\bB^{\top}\bx_{i}) \|^{2}_{2},   
\end{align*}
We use the batch-training strategy to train the neural networks. We use Adam optimizer with Pytorch default parameters. We choose the batch size to be 0.5\% of the sample size and train the neural networks for 200 epochs.
\subsection{Score Functions}\label{sec:app-sf}
We assume that the distributions are all non-degenerated, i.e., dispersion matrices are all positive definite. Let $Q_{(\cdot)}(\bx)$ denote $(\bx - \mu_{(\cdot)})^{\T} \bSigma^{-1}_{(\cdot)}(\bx - \mu_{(\cdot)})$.
\begin{itemize}
    \item [1. ] For $\bx \sim \mathcal{N} (\bmu_{\mathcal{N}}, \bSigma_{\mathcal{N}})$, $P(\bx) = (2\pi)^{-p/2}|\bSigma_{\cN}|^{-1/2}\exp\{ -1/2 \cdot Q_{\cN}(\bx)
 \}$.
     \begin{align*}
        &\bs(\bx) = -\nabla_{\bx} \ln\{P(\bx)\} = \frac{1}{2} \nabla_{\bx}\{Q_{\cN}(\bx) + p\ln{(2\pi)} + \ln(|\bSigma_{\mathcal{N}}|)\} = \bSigma^{-1}_{\mathcal{N}}(\bx-\bmu_{\mathcal{N}}).\\   
        & \bT(\bx) = \bs(\bx)\bs(\bx)^{\top} - \nabla_{\bx}\bs(\bx) = \bSigma^{-1}_{\mathcal{N}}(\bx-\mu_{\mathcal{N}})(\bx-\mu_{\mathcal{N}})^{\top}\bSigma^{-1}_{\mathcal{N}} - \bSigma^{-1}_{\mathcal{N}}.
 \end{align*} 
    \item [2. ] For $\bx \sim  t_{\nu}(\bmu_{t}, \bSigma_{t})$, $\nu \geq 2$, $P(\bx) = (\nu-2)^{\nu/2}\Gamma\{ (\nu+p) /2\}/\pi^{p/2}/|\bSigma_{t}|^{1/2}/\Gamma(\nu/2)/\{\nu-2 + Q_{t}(\bx)\}^{(\nu+p)/2}$.
     \begin{align*}
    \bs(\bx) = &  -\nabla_{\bx} \ln\{P(\bx)\}\\
     = & \nabla_{\bx} \left(\frac{\nu + p}{2}\ln \left \{ \nu - 2 + Q_{t}(\bx)\right\}  - \ln \left[ \frac{(\nu-2)^{\nu/2}\Gamma\{ (\nu + p)/2\}}{\Gamma(\nu/2)\pi^{p/2}|\bSigma_{t}|^{1/2}} \right ]\right )\\
    = & \frac{ (p + \nu)\bSigma^{-1}_{t}(\bx - \bmu_{t})}{ \nu - 2 + Q_{t}(\bx)}.\\
     \bT(\bx) = & \bs(\bx)\bs(\bx)^{\top} - \nabla_{\bx}\bs(\bx) \\
    = & \frac{ (p + \nu)(p + \nu + 2)\bSigma^{-1}_{t}(\bx - \mu_{t})(\bx - \mu_{t})^{\top}\bSigma^{-1}_{t} - (p+\nu)\{ \nu -2  +Q_{t}(\bx)\}\bSigma^{-1}_{t}}{ \{ \nu - 2 + Q_{t}(\bx)\}^{2}}.
 \end{align*}
 \item [3. ] For $\bx \sim \cH_{\chi,\psi}$, 
 \begin{align*}
    P(\bx) = \frac{(\sqrt{\psi/\chi})^{(p+1)/2}}{(2\pi)^{p/2}|\bSigma|^{1/2}(\psi)^{1/2}K_{1/2}(\sqrt{\chi\psi})} \times K_{1/2}(\sqrt{\{ \chi + Q(\bx) \}\psi}[ \sqrt{\{ \chi + Q(\bx) \}\psi} ]^{1/2},
\end{align*}
where $K_{1/2}(x) = \sqrt{\pi/(2x)}\exp(-x)$.
\begin{align*}
    \bs(\bx) = & -\nabla_{\bx} \ln\{P(\bx)\} = \nabla_{\bx} \sqrt{\{\chi + Q_{\cH}(\bx)\}\psi} = \frac{\psi^{1/2}\bSigma^{-1}_{\cH}(\bx - \mu_{\cH})}{\sqrt{\{ \chi +Q_{\cH}(\bx)\}}}.\\
   \bT(\bx) = & \bs(\bx)\bs(\bx)^{\top} - \nabla_{\bx}\bs(\bx) \\
    = & \frac{\{\psi + \psi^{1/2}/\sqrt{\chi + Q_{\cH}(\bx)}\}\bSigma^{-1}_{\cH}(\bx - \mu_{\cH})(\bx - \mu_{\cH})^{\top}\bSigma^{-1}_{\cH} - \psi^{1/2}\sqrt{\chi + Q_{\cH}(\bx)}\}\bSigma^{-1}_{\cH}}{\{ \chi + Q_{\cH}(\bx)\}}
\end{align*}
\end{itemize}

\section{Details of Real Data Analysis}
\subsection{Details on MNIST Dataset}
\subsubsection{Structural Similarity Index Measure}\label{sec-app-ssim}
Given two images $\bX$ and $\bY$, 
    \begin{align}
        \text{SSIM} \coloneqq \frac{(2\mu_{x}\mu_{y} + c_{1})(2\sigma_{xy} +c_{2})}{(\mu^{2}_{x}+\mu^{2}_{y} + c_{2})(\sigma^{2}_{x} + \sigma^{2}_{y} + c_{2})},
    \end{align}
    where $(\mu_{x}, \mu_{y})$ and $(\sigma_{x}, \sigma_{y})$ are means and variances of pixel values of $\bX$ and $\bY$ respectively, $\sigma_{xy}$ is the covariance of pixel values of $\bX$ and $\bY$, $c_{1}=(0.01R)^{2}$, $c_{2} = (0.03R)^{2}$, where $R$ denotes the range of pixel values in
the image. 

\subsubsection{Structure and Training Details of 
  $\cC_{\cE}(\cdot)$}\label{sec:app-cl}

$\cC_{\cE}(\cdot)$ is a one layer neural network composed with softmax function, i.e., $\cC_{\cE}(\bz) = \text{softmax}( \bW^{\top}_{\cE}\bz + \bb_{\cE})$, where $\bW_{\cE} \in \RR^{h\times 10}$ and $\bb_{\cE} \in \RR^{10}$ are the weight matrix and bias vector corresponding to the encoder, respectively, which are estimated by minimizing the following loss function:
\begin{align*}
    \ell(\bW_{\cE}, \bb_{\cE}) = \EE[ \|\text{softmax}\{\bW^{\top}_{\cE}\cE(\bX) + \bb_{\cE}\} - \by^{\ddag} \|^{2}_{2}],
\end{align*}
where $\by^{\ddag}$ is the one-hot encoding of $\by$.  We use the batch-training strategy to train the neural network. We use Adam optimizer with
Pytorch default parameters. We choose the batch size to be 128 and we train the neural network for 50 epochs.

\subsubsection{Structure and Training Details of Autoencoder 
 }\label{sec:app-st}
Table \ref{tab:ae} demonstrates the structure of the AE we use. The same decoder structure is also used for our method. To estimate the parameters of AE $(\btheta_{e}, \btheta_{d})$, we minimize the loss function \eqref{eq:loss-ae}; and to estimate the parameters of decoders of our methods, we minimize the loss functions \eqref{eq:def-sdoc}. For both AE and our methods, we use the batch-training strategy to train the neural
network. We use Adam optimizer with Pytorch default parameters. We choose the batch size to
be 128 and we train the neural network for 50 epochs.

\begin{table}[H]
\caption{Structure of the Autoencoder}
\vskip 0.15in
\centering
\begin{tabular}{clllll}
\hline
& Layer & \multicolumn{3}{l}{}                   & Type                                                 \\ \hline
\multirow{5}{*}{Encoder}    & 1     & \multicolumn{3}{l}{\multirow{10}{*}{}} & fully connected layer (in dims = 784, out dims = 4h) \\ \cline{2-2} \cline{6-6} 
                            & 2     & \multicolumn{3}{l}{}                   & ReLU     \\ \cline{2-2} \cline{6-6} 
                            & 3     & \multicolumn{3}{l}{}                   & fully connected layer (in dims = 4h, out dims = 2h)  \\ \cline{2-2} \cline{6-6} 
                            & 4     & \multicolumn{3}{l}{}                   & ReLU                                                 \\ \cline{2-2} \cline{6-6} 
                            & 5     & \multicolumn{3}{l}{}                   & fully connected layer (in dims = 2h, out dims = h)   \\ \cline{1-2} \cline{6-6} 
\multicolumn{1}{l}{\multirow{5}{*}{Decoder}} & 1     & \multicolumn{3}{l}{}         & fully connected layer (in dims = h, out dims = 2h)   \\ \cline{2-2} \cline{6-6} 
\multicolumn{1}{l}{}        & 2     & \multicolumn{3}{l}{}                   & ReLU                                                 \\ \cline{2-2} \cline{6-6} 
\multicolumn{1}{l}{}        & 3     & \multicolumn{3}{l}{}                   & fully connected layer (in dims = 2h, out dims = 4h)  \\ \cline{2-2} \cline{6-6} 
\multicolumn{1}{l}{}        & 4     & \multicolumn{3}{l}{}                   & ReLU                                                 \\ \cline{2-2} \cline{6-6} 
\multicolumn{1}{l}{}        & 5     & \multicolumn{3}{l}{}                   & fully connected layer (in dims = 4h, out dims = 784) \\ \hline
\end{tabular}
\label{tab:ae}
\vskip -0.1in
\end{table}
\subsubsection{Comparison of Performances of Competing Methods}
We provide Figure~\ref{fig:app-mes}, the enlarged version of Figure~\ref{fig:mes} in Section~\ref{sec:rda}, in the main paper. 
\begin{figure}[h]
\centering    \includegraphics[width=\textwidth]{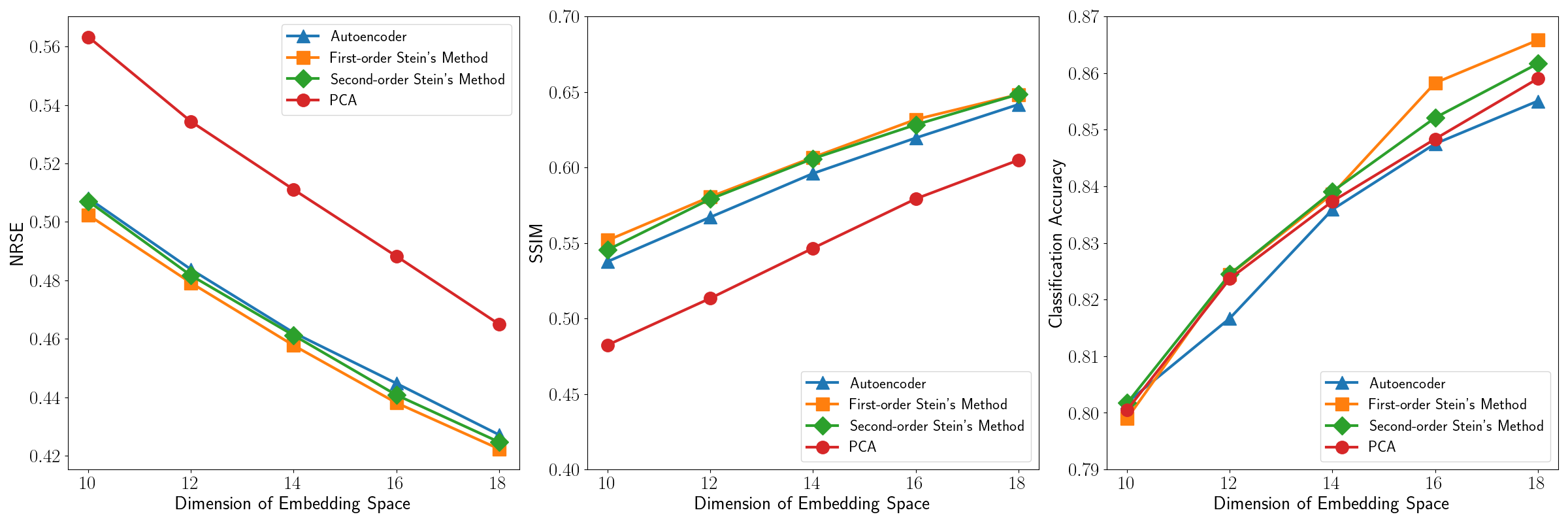}
\caption{Comparison of qualities of learned latent spaces of competing methods based on three different metrics. The values are calculated empirically on the testing set and medians over 100 repetitions are reported. From Left to right, the metrics are NMSE, SSIM and classification accuracy, respectively.}\label{fig:app-mes}
\end{figure}

\subsection{Details on M1 Patch-seq Dataset}\label{sec:app-rda2}
We choose $n_{1} = 213$, $n_{2} \in \{ 50, 75, 100\}$, to make fair comparison for different $n_{2}$, we let $n_{3} = 900$.  

To imply our first-order method, we use the $\nu$-method introduced in \citet{zhou2020} as our score estimator $\hat{\bs}(\cdot)$. We let $\lambda = e^{-5}$ and use the curl-free version of IMQ kernel. 

We calculate the PCA estimator on $(\bx_{n_{1}+1}, \ldots, \bx_{n})$.  

We provide Figure~\ref{fig:app-rd2}, the enlarged version of Figure~\ref{fig:rda2} in Section~\ref{sec:app-rda2}, in the main paper. 
\begin{figure}[h]
\centering    \includegraphics[width=\textwidth]{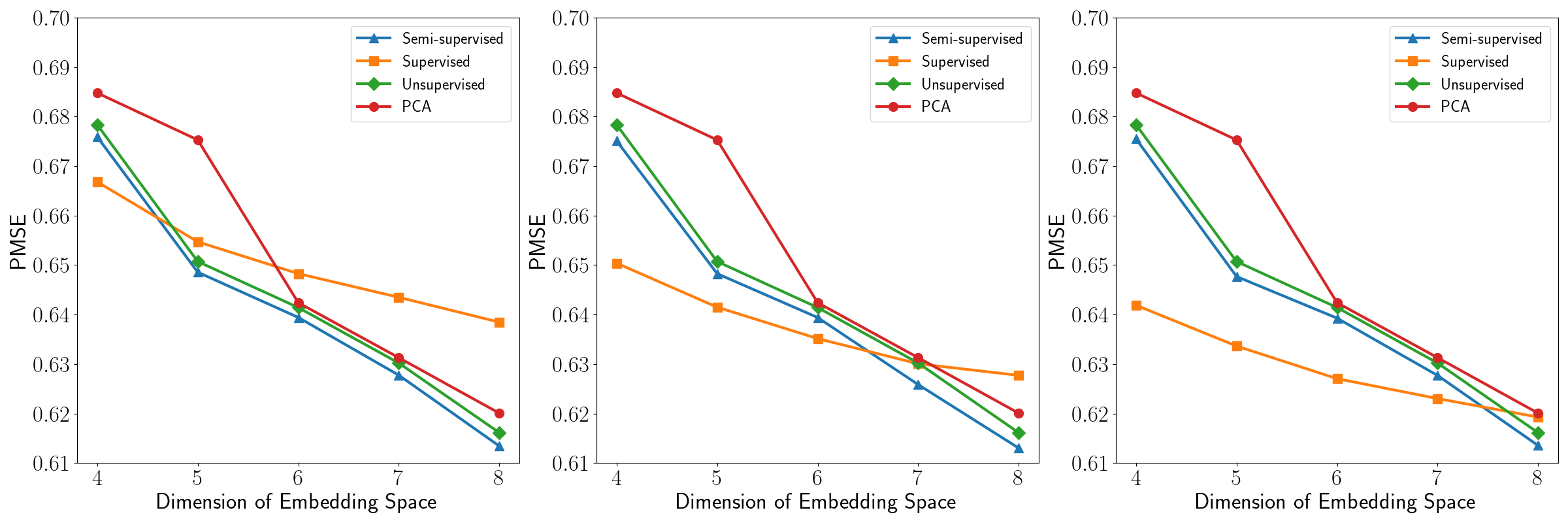}
\caption{Comparison of qualities of learned latent spaces of competing methods based on three different metrics. The values are calculated empirically on the testing set and medians over 200 repetitions are reported. From left to right, the label size $n_{2}$ are 50, 75, 100, respectively.  }\label{fig:app-rd2}
\end{figure}

\section{Proofs of Main Results}  

\subsection{Proofs of Stein's Lemmas}

\subsubsection{Proof of Lemma~\ref{lem:fs}}
\begin{proof}
For $\forall i \in [p], j \in [q]$, 
\begin{align*}
   & \EE\{\by_{j} \bs_{i}(\bx)\} = \EE[ \{f_{j}(\bB^{\T}\bx) +\epsilon_{j}\}\bs_{i}(\bx)] = \EE\{f_{j}(\bB^{\T}\bx)\bs_{i}(\bx)\}  + \EE\{\epsilon_{j}\bs_{i}(\bx) \} =  \EE\{ f_{j}(\bB^{\T}\bx)\bs_{i}(\bx)\}\\
    = & -\int_{\RR^{p}}  f_{j}(\bB^{\T}\bx) \bs_{i}(\bx)P(\bx)\dx = -\int_{\RR^{p}} f_{j}(\bB^{\T}\bx) \nabla_{\bx_{i}}P(\bx)\dx\\
    = & \int_{\RR^{p}} \nabla_{\bx_{i}} f_{j}(\bB^{\T}\bx) P(\bx)\dx - \Delta\\
    = & \bb^{\T}_{i} \EE [\nabla_{\bz} \{ f_{j}(\bB^{\T}\bx)\}],
\end{align*}   
where $\bb_{i}$ is the $i$-th column of $\bB^{\T}$ and
\begin{align*}
   \Delta = &\int_{\RR^{p-1}}  \{\lim_{a \rightarrow \infty} f_{j}(\bB^{\T}\bx) P(\bx)|_{\bx = (x_{1}, \ldots,x_{i-1},a,  x_{i+1}, \ldots, x_{p}) } - \lim_{b \rightarrow -\infty} f_{j}(\bB^{\T}\bx) P(\bx)|_{\bx = (x_{1}, \ldots,x_{i-1},b,  x_{i+1}, \ldots, x_{p}) }\}\\
  & \text{d}(x_{1}, \ldots,x_{i-1}, x_{i+1}, \ldots, x_{p}) =  0.
\end{align*}

Stacking $\EE\{ \bs_{i}(\bx) \by_{j} \}, i \in [p]$ together, we have $\EE\{ y_{j}\bs(\bx) \} =  \bB \EE [\nabla_{\bz} \{ f_{j}(\bB^{\T}\bx)\}]$. 
\end{proof}

\subsubsection{Proof of Lemma~\ref{lem:ss}}
\begin{proof}
For $\forall j \in [q]; \; k,l \in [p]$, 
\begin{align*}
& \EE\{ y_{j}  \bT_{k,l}(\bx)\} =  \EE\{ f_{j}(\bB^{\T} \bx) \bT_{k,l}(\bx)\} = \EE \left \{ \frac{f_{j}(\bB^{\T} \bx)}{P(\bx)} \nabla^{2}_{\bx_{k},\bx_{l}}P(\bx)\right \} \\
= & \int_{\RR^{p}} \left \{ f_{j}(\bB^{\T} \bx) \nabla^{2}_{\bx_{k},\bx_{l}}P(\bx)\right \} \dx = -\int_{\RR^{p}} \left \{\nabla_{\bx_{k}} f_{j}(\bB^{\T} \bx)  \nabla_{\bx_{l}}P(\bx)\right \} \dx + \Delta_{1}\\
= &   -\int_{\RR^{p}} \left \{ \nabla_{\bx_{k}}f_{j}(\bB^{\T} \bx) \nabla_{\bx_{l}} P(\bx)\right \} \dx = - \bb^{\T}_{k} \int_{\RR^{p}} \left \{ \nabla_{\bz} f_{j}(\bB^{\T} \bx)  \nabla_{\bx_{k}}f_{j}(\bB^{\T} \bx)\right \} \dx \\
= & -\bb_{k}^{\T} \Delta_{2} + \bb_{k}^{\T} \int_{\RR^{p}} \left [P(\bx)  \nabla_{\bx_{l}} \{\nabla_{\bz} f_{j}(\bB^{\T} \bx)\}\right ] \dx\\
= & \bb_{k}^{\T} \int_{\RR^{p}} \left [P(\bx)  \nabla_{\bx_{l}} \{\nabla_{\bz} f_{j}(\bB^{\T} \bx)\}\right ] \dx = \bb_{k}^{\T}\int_{\RR^{p}} \left  \{P(\bx) \nabla^{2}_{\bz} f_{j}(\bB^{\T} \bx)\right \} \dx\bb_{l} \\
= & \bb_{k}^{\T} \EE \{ \nabla^{2}_{\bz} f_{j}(\bB^{\T} \bx)\}\bb_{l}, 
\end{align*}
where $\bb_{k}$ and $\bb_{l}$ are the $k$-th and $l$-th columns of $\bB^{\T}$.
\begin{align*}
    \Delta_{1} = &  \int_{\RR^{p-1}} \left [ \lim_{a \rightarrow +\infty} \{ f_{j}(\bB^{\T} \bx) \nabla_{\bx_{l}} P(\bx)\}|_{x_{1}, \ldots, x_{k-1},a,x_{k+1}, \ldots, x_{p} } - \lim_{b \rightarrow -\infty} \{ f_{j}(\bB^{\T} \bx) \nabla_{\bx_{l}} P(\bx)\}|_{x_{1}, \ldots, x_{k-1},b,x_{k+1}, \ldots, x_{p} }\right ] \\
    &\text{d}(x_{1}, \ldots, x_{k-1},x_{k+1}, \ldots, x_{p} )=0,
\end{align*}
and $\Delta_{2}$ is an $r$-dimensional vector, the $i$-th element of which is 
\begin{align*}
&\int_{\RR^{p-1}} \left [ \lim_{a \rightarrow +\infty} \{ \nabla_{\bz_{i}} f_{j}(\bB^{\T} \bx) P(\bx)\}|_{x_{1}, \ldots, x_{l-1},a,x_{l+1}, \ldots, x_{p} } - \lim_{b \rightarrow -\infty} \{ \nabla_{\bz_{i}} f_{j}(\bB^{\T} \bx) P(\bx)\}|_{x_{1}, \ldots, x_{l-1},b,x_{l+1}, \ldots, x_{p} }\right ] \\
    &\text{d}(x_{1}, \ldots, x_{l-1},x_{l+1}, \ldots, x_{p} )=0.
\end{align*}

Stacking all $\EE\{ y_{j}  \bT_{k,l}(\bx)\}$, $k,l \in [p]$ together, we have $\EE\{ y_{j}  \bT(\bx)\} = \bB \EE \{ \nabla^{2}_{\bz} f_{j}(\bB^{\T} \bx)\}\bB^{\T}$. 
\end{proof}

\subsection{Proofs of the First-order Method}
In the following proofs, absolute constants $C_{1}$ to $C_4$ are from Assumptions~\ref{assum:1},~\ref{assum:2},~\ref{assum:4}, and~\ref{assum:taylor}, respectively.

We will use the following simple lemma without proof. 
\begin{lemma}\label{lem:gaussian-score}
    Suppose $\bx \sim \cN(0, \bSigma)$, the first-order score function has the form $\bs(\bx) = \bSigma^{\dagger}\bx$. 
\end{lemma}
Empirically, if we have observations $\bx_1, \ldots, \bx_{n} \overset{\text{i.i.d.}}{\sim} \cN(0, \bSigma)$, denote the sample variance matrix $\widehat{\bSigma} = 1/n\sum_{i=1}^{n}\bx_i \bx_i^\T$, then, the plug-in estimator of $\bs(\bx_{i})$ can be $\hat{\bs}(\bx_{i}) = \widehat{\bSigma}^{\dagger}\bx_{i}$.

\subsubsection{Proof of Theorem~\ref{thm:pca}}
Based on Lemma~\ref{lem:gaussian-score}, we prove Theorem~\ref{thm:pca}.    
\begin{proof}
Let $\widehat{\bSigma} = (1/n) \sum_{i=1}^{n}  \bx \bx_{i}^{\top}$.
Notice that in the statement of the theorem, by `` the column space learned
from equation~\eqref{eq:def-fe}'', we refer to the case where $\bSigma$ is unknown and replaced by $\widehat{\bSigma}$, so that $\hat{\bs}(\bx) = \widehat{\bSigma}^{\dagger} \bx$. 
We assume that $\widehat{\bSigma}$ has the eigenvalue decomposition $\widehat{\bSigma} = \bU \Tilde{\bLambda} \bU^{\T}$, where
\begin{align*}
\Tilde{\bLambda} =  \begin{pmatrix}
 \Lambda&  \\
 & 0 \\
\end{pmatrix}, 
\end{align*}
$\bLambda$ is a diagonal and invertible matrix. By 
 equation~\eqref{eq:def-fe} and Lemma~\ref{lem:gaussian-score}, we have
\begin{align*}
\widehat{\bB} = \text{SVD}_{l,r}\left \{ \frac{1}{n} \sum_{i=1}^{n}  \bs(\bx_{i}) \by_{i}^{\top} \right \} =  \text{SVD}_{l,r}\left (\widehat{\bSigma}^{\dagger} \widehat{\bSigma} \right ) =  \text{SVD}_{l,r}\left \{ \bU \begin{pmatrix}
 \bI_{p}&  \\
 & 0 \\
\end{pmatrix} \bU^\T \right \} = \bU[:,:r],
\end{align*}
which means that columns of $\widehat{\bB}$ consist of first r leading
principal components of $\widehat{\bSigma}$. 
   
\end{proof}

\subsubsection{Proof of Theorem \ref{thm:1}}
We first prove a lemma that establishes a bound on the moment of a Gaussian random variable.
\begin{lemma}\label{lemma:GaussianMoment}
    If random variable $z \sim \cN(0, \sigma^2)$, then $[\EE|z|^\ell]^{\frac{1}{\ell}} \leq \sigma \sqrt{\ell} $ for $\ell \geq 1$.
\end{lemma}
\begin{proof}
The $\ell$-th moment of $z$ can be expressed as follows
\begin{align*}
\EE|z|^\ell &= \int_\RR |z|^\ell\frac{1}{\sqrt{2\pi}\sigma}\mathrm{exp}\left(-\frac{z^2}{2\sigma^2}\right)dz\\
&= 2\int_{0}^{\infty} z^\ell\frac{1}{\sqrt{2\pi}\sigma}\mathrm{exp}\left(-\frac{z^2}{2\sigma^2}\right)dz\\
&\overset{(i)}{=} 2\int_{0}^{\infty} \frac{1}{\sqrt{2\pi}\sigma} (\sqrt{2t}\sigma)^{\ell} e^{-t} \cdot \frac{\sqrt{2}\sigma}{2\sqrt{t}}dt\\
&= \frac{(\sqrt{2}\sigma)^\ell}{\sqrt{\pi}} \cdot \int_{0}^{\infty} t^\frac{\ell-1}{2}e^{-t}dt\\
&\overset{(ii)}{=} \frac{(\sqrt{2}\sigma)^\ell}{\sqrt{\pi}} \cdot \Gamma\left(\frac{\ell+1}{2}\right),
\end{align*}
where $(i)$ is by the transformation $z=\sqrt{2t}\sigma$, $(ii)$ is by $\Gamma(x)=\int_0^{\infty}t^{x-1}e^{-t}dt$. If $\ell$ is even, then we have, 
\begin{align*}
\frac{(\sqrt{2}\sigma)^\ell}{\sqrt{\pi}} \cdot \Gamma\left(\frac{\ell+1}{2}\right) &= \frac{(\sqrt{2}\sigma)^\ell}{\sqrt{\pi}} \cdot \frac{\ell-1}{2}\cdot\frac{\ell-3}{2},\ldots, \cdot\frac{1}{2}\cdot\Gamma\left(\frac{1}{2}\right)\leq \frac{(\sqrt{2}\sigma)^\ell}{\sqrt{\pi}} \cdot \frac{\sqrt{\ell!}}{2^{\ell/2}}\cdot\sqrt{\pi}= \sigma^\ell\sqrt{\ell!}.
\end{align*}
On the other hand, if $\ell$ is odd, then we get
\begin{align*}
\frac{(\sqrt{2}\sigma)^\ell}{\sqrt{\pi}} \cdot \Gamma\left(\frac{\ell+1}{2}\right) &= \frac{(\sqrt{2}\sigma)^\ell}{\sqrt{\pi}} \cdot \frac{\ell-1}{2}\cdot\frac{\ell-3}{2},\ldots, \cdot1\cdot\Gamma(1)\leq \frac{(\sqrt{2}\sigma)^\ell}{\sqrt{\pi}} \cdot \frac{\sqrt{\ell!}}{2^{(\ell-1)/2}}\leq \sigma^\ell\sqrt{\ell!}.
\end{align*}
In summary, we arrive at
\begin{align*}
[\EE|z|^\ell]^\frac{1}{\ell} \leq (\sigma^\ell\sqrt{\ell!})^\frac{1}{\ell} \leq \sigma\sqrt{\ell}.    
\end{align*}
\end{proof}
Then, we can obtain the following concentration property of $\bs(\bx) \by^{\T}$.  
\begin{lemma}\label{lemma:1} Suppose model (\ref{model:1.2}) holds, under Assumptions~\ref{assum:1} to \ref{assum:3}, with probability at least $1 - \delta$,
    \begin{align*}        
        \left \|\frac{1}{n}\sum_{i=1}^{n} \bs(\bx_i)\by^{\T}_i - \EE[\bs(\bx)\by^\T] \right \|_F = \cO\left(\sqrt{\frac{pq}{n}}\sqrt{\ln\frac{pq}{\delta}}\right).
    \end{align*}
\end{lemma}
\begin{proof}
    We first prove that $\by_j\bs_k(\bx) - \EE \by_j\bs_k(\bx)$ follows the sub-exponential distribution. For simplicity, we denote $y := \by_j$, $s = \bs_k(\bx)$, $\epsilon := \bepsilon_j$ and $f := f_j(\bB^\T \bx)$. Then for $\ell \geq 1$, we have 
    \begin{align*}
    [\EE |ys - \EE(ys)|^\ell]^{\frac{1}{\ell}} 	& \overset{(i)}{=}   [\EE |(f + \epsilon)s - \EE(fs)|^\ell]^{\frac{1}{\ell}}\\ 
   &  \overset{(ii)}{\leq} [\EE |fs|^\ell]^{\frac{1}{\ell}} + [\EE |\epsilon s|^\ell]^{\frac{1}{\ell}} + |\EE (fs)|\\
   &  \overset{(iii)}{\leq}  [\EE |f|^{2\ell}]^{\frac{1}{2\ell}} \cdot [\EE |s|^{2\ell}]^{\frac{1}{2\ell}} + [\EE |\epsilon|^{\ell}]^{\frac{1}{\ell}} \cdot [\EE |s|^{\ell}]^{\frac{1}{\ell}} + [\EE |f|^{2}]^{\frac{1}{2}} \cdot [\EE |s|^{2}]^{\frac{1}{2}}\\
   & \overset{(iv)}{\leq} C_2 \cdot C_1 \sqrt{2\ell} + \sigma_{\epsilon}\sqrt{\ell} \cdot C_1\sqrt{\ell} + C_2\cdot C_1\sqrt{2}\\
   & \overset{(v)}{\leq} 3 C_1C_2\ell+ \sigma_\epsilon C_1\ell\\
   & = (3C_2 + \sigma_\epsilon) C_1\ell
   \end{align*}
    where $(i)$ is by independent, $(ii)$ is by Minkowski's inequality. $(iii)$ holds by Cauchy-Schwarz inequality and $(iv)$ is by the Assumption \ref{assum:1}, \ref{assum:2} and Lemma \ref{lemma:GaussianMoment}. In addition, $(v)$ holds since that $\ell \geq 1$.

    Therefore, each element of $ \bs(\bx)\by^{\T}$ follows the sub-exponential distribution. Hence, by the properties of sub-exponential, we have 
    \begin{align*}
   \|ys - \EE(ys)\|_{\psi_1} &= \sup_{\ell \geq 1} \frac{1}{\ell} [\EE |ys - \EE(ys)|^\ell]^{\frac{1}{\ell}}\\
   & \leq (3C_2 + \sigma_\epsilon) C_1
\end{align*}
On the other hand, according to Bernstein's inequality, for any $t >0$, we have
\begin{align*}
\PP \left(|\frac{1}{n}\sum_{i = 1}^{n} y_i s_i - \EE (ys)|\geq t \right) \leq 2\cdot \exp\left\{ -cn \min \left( \frac{t^2}{K^2}, \frac{t}{K}\right)\right\}
\end{align*}
where $K \leq (3C_2 + \sigma_\epsilon) C_1$ and $c>0$ is an absolute constant. Recall that the above provides a tail bound for the $(j,r)$-th element of $ \bs(\bx)\by^{\T}$. Next, we consider the tail bound for the entire matrix $ \bs(\bx)\by^{\T}$ as follows
\begin{align*}
& \PP \left[\|\frac{1}{n}\sum_{i = 1}^{n} \bs(\bx_i)\by_i^{\T} - \EE \{\bs(\bx)\by^{\T}\}\|_F /\sqrt{pq} \geq t \right]\\
& \leq \sum_{j = 1}^q \sum_{r = 1}^p  \PP \left\{|\frac{1}{n}\sum_{i = 1}^{n} y_i s_i - \EE (ys)|\geq t \right\}\\ 
& \leq 2 pq \cdot \exp\left\{ -cn \min \left( \frac{t^2}{K^2}, \frac{t}{K}\right)\right\}
\end{align*}
Letting $t=(3C_2 + \sigma_\epsilon) C_1\cdot\sqrt{\ln(2pq/\delta)}/\sqrt{cn}$, we get
\begin{align*}
\PP \left[\left\|\frac{1}{n}\sum_{i = 1}^{n} \bs(\bx_i)\by_i^\T - \EE \{ \bs(\bx)\by^\T\}\right\|_F \geq \sqrt{\frac{pq}{cn}} (3C_2+\sigma_\varepsilon)C_1\sqrt{\ln\frac{2pq}{\delta}}\right] \leq \delta
\end{align*}
Therefore, with probability at least $1-\delta$, we have
\begin{align*}
\left\|\frac{1}{n}\sum_{i = 1}^{n} \bs(\bx_i)\by_i^\T  - \EE\{ \bs(\bx)\by^\T\}\right\|_F = \cO\left(\sqrt{\frac{pq}{n}}\sqrt{\ln\frac{pq}{\delta}}\right).
\end{align*}
\end{proof}
We are now prepared to prove the main theory on the first order estimator $\widehat{\bB}$, as defined in equation (\ref{eq:def-fe}).
\begin{proof}[\textbf{Proof of Theorem~\ref{thm:1}}]
According to the result of Lemma \ref{lemma:1}, with probability at least $1 - \delta$, we have
\begin{align*}
\left\|\frac{1}{n}\sum_{i = 1}^{n} \bs(\bx_i) \by_i^\T- \EE \{\bs(\bx)\by^\T\}\right\|_F = \cO\left(\sqrt{\frac{pq}{n}}\sqrt{\ln\frac{pq}{\delta}}\right).
\end{align*}
Next, by the Theorem 2 in  \cite{yu2015useful} and Theorem 19 in \cite{o2018random}, there exists an orthogonal matrix $\hat{\bV} \in \RR^{r\times r}$, such that  
\begin{align*}
\|\bB-\widehat{\bB} \hat{\bV}\|_F & \leq  2^{\frac{3}{2}} \cdot \frac{\left\|\frac{1}{n}\sum_{i = 1}^{n} \bs(\bx_i) \by_i^\T  - \EE \{\bs(\bx)\by^\T \}\right\|_F}{\sigma_r(\bM_1)}\\
& = \cO\left(\sqrt{\frac{rp}{n}}\sqrt{\ln\frac{pq}{\delta}}\right).
\end{align*}
where the last equality holds by Assumption \ref{assum:3}.
\end{proof}

\subsubsection{Proof of Theorem~\ref{thm:2}}
{To begin, we present three lemmas that will be used in the subsequent proofs.}
\begin{lemma}\label{lemma:covariance} ( \citealt{vershynin2018high}, Covariance estimation).
	Let $\bX$ be an $n\times p$ matrix whose rows $\bx_i$ are independent, mean zero, sub-Gaussian random
	vectors in $\mathbb{R}^p$ with the same covariance matrix $\bSigma$. Denote $\widehat{\bSigma} = (1/n)\bX^\T\bX$ as the sample covariance matrix. Then for any $t\ge 0$, with probability at least $1 - 2\exp(-t^2p)$, we have
	\begin{align*}
		\left\|\widehat{\bSigma} - \bSigma\right\|_{2} \le \max(\delta, \delta^2),
	\end{align*}
where $\delta=CLt\sqrt{p/n}$, $L=\max(K, K^2)$ and $K = \max_i\|\bx_i\|_{\psi_2}$.
\end{lemma}

\begin{lemma}\label{lemma:GaussianSingularbound} (\citealt{vershynin2018high}, Random matrices with general rows)
	Let $\bX$ be an $n\times p$ matrix whose rows $\bx_i$ are independent isotropic random vectors in $\mathbb{R}^p$. Assume that for some $L\ge 0$, it holds almost surely for every $i$ that $\|\bx_i\|_2 \le L\sqrt{p}$. Then for every $t\ge 0$, one has with probability at least $1 - 2p\cdot\exp(-ct^2)$,
    \begin{align*}
        \sqrt{n} - tL\sqrt{p} \le \sigma_{\min}(\bX) \le \sigma_{\max}(\bX) \le \sqrt{n} + tL\sqrt{p},
    \end{align*}
     where $\sigma_{\min}(\bX)$ and $ \sigma_{\max}(\bX)$ are the smallest and largest singular values of $\bX$, respectively. 
\end{lemma}

\begin{lemma}\label{lemma:inverse} (\citealt{xu2020perturbation}, Perturbation of inverse)
	For two matrices $\bW, \bQ\in \mathbb{R}^{p\times p}$, we have
	\begin{align*}
		\left\|\bW^{-1} - \bQ^{-1}\right\|_2 \le\|\bW\|_2\|\bQ\|_2\left\|\bW - \bQ\right\|_2.
	\end{align*}
\end{lemma}

\begin{lemma}\label{lemma:GaussianNorm}
	Let $\bx \in \mathbb{R}^p$ be a sub-Gaussian random vector with parameter $\sigma$, then with probability at least $1 - t$ for $t \in (0, 1)$,
	\begin{align*}
		\|\bx\|_2 \le 4\sigma\sqrt{p} + 2\sigma \sqrt{\ln(1/t)}.
	\end{align*}
\end{lemma}
\begin{proof}[\textbf{Proof of Theorem~\ref{thm:2}}] According to Taylor's theorem, $f_j(\bB^\T \bx)$ can be expanded at the point $\ba = \boldsymbol{0}_r$ as follows with a function $h_j:\mathbb{R}^{r} \rightarrow \mathbb{R}$
\begin{align*}
    f_j(\bB^\T \bx) &= \langle\nabla f_j(\boldsymbol{0}_{r}), \bB^\T \bx\rangle + h_j(\bB^\T \bx)\|\bB^\T \bx \|_2\\
    &= \bx^\T \bB \nabla f_j(\boldsymbol{0}_{r})^\T  + h_j(\bB^\T\bx)\|\bB^\T\bx\|_2,
\end{align*}
with $\lim_{\bB^\T \bx\rightarrow \boldsymbol{0}_r} h_j(\bB\bx)=0$ for all $j \in [q]$. Hence, we arrive at
\begin{align*}
\bF(\bB^\T \bx)^\T  = \bx^\T \bB \cdot  \nabla\bF(\boldsymbol{0}_r)^\T + \bH(\bB^\T \bx)^\T \|\bB^\T\bx\|_2,
\end{align*}
where $\nabla\bF (\boldsymbol{0}_r) := \left\{\nabla f_1(\boldsymbol{0}_{r}), \ldots, \nabla f_q(\boldsymbol{0}_{r})\right\}^\T \in \RR^{q\times r}$ and $\bH\{\bB^\T \bx) := \left(h_1(\bB\bx), \ldots, h_q(\bB\bx) \right\}^\T \in \RR^{q}$. Therefore, we have 
\begin{align*}
 \frac{1}{n}\sum_{i = 1}^{n} \bs(\bx_i) \by_i^\T =  \frac{1}{n}\sum_{i = 1}^{n} \bs(\bx_i)\bx_i^\T \bB \cdot  \nabla\bF (\boldsymbol{0}_r)^\T +  \frac{1}{n}\sum_{i = 1}^{n} \bs(\bx_i) \left( \bepsilon_i^\T + \bH(\bB^\T \bx_i)^\T \|\bB^\T\bx_i\|_2 \right).
\end{align*}
For Gaussian distribution with zero mean and covariance matrix $\bSigma$, we have the score of $\bx$ is $\bs(\bx) = \bSigma^{\dagger}\bx$. Denote $\widehat{\bSigma} := (1/n) \sum_{i=1}^{n}\bx_i \bx_i^\T$. Thus, if the covariance matrix is unknown, i.e., the $\bSigma$ replaced by $\widehat{\bSigma}$, 
we then have 
\begin{align}\label{pf-eq-terms}
& \left\|\frac{1}{n}\sum_{i = 1}^{n} \bs(\bx_i) \by_i^\T - \frac{1}{n}\sum_{i = 1}^{n} \hat{\bs}(\bx_i) \by_i^\T \right\|_F \nonumber\\
 &= \left\| \frac{1}{n}\sum_{i = 1}^{n} (\bSigma^{\dagger} - \widehat{\bSigma}^{\dagger})\bx_i\bx_i^\T \bB \cdot  \nabla\bF (\boldsymbol{0}_r)^\T +  \frac{1}{n}\sum_{i = 1}^{n} (\bSigma^{\dagger} - \widehat{\bSigma}^{\dagger})\bx_i \left( \bepsilon_i^\T + \bH(\bB^\T \bx_i)^\T \|\bB^\T\bx_i\|_2 \right)\right\|_F\nonumber\\
& \leq {\|(\bSigma^{\dagger}\widehat{\bSigma} - \bI_{p}) \bB \cdot  \nabla\bF (\boldsymbol{0}_r)^\T\|_F} + {\frac{1}{n} \|(\bSigma^{\dagger} -  \widehat{\bSigma}^{\dagger}) \sum_{i = 1}^{n}\left( \bx_i \epsilon_i^\T +  \bx_i\bH(\bB^\T \bx_i)^\T \|\bB^\T\bx_i\|_2  \right)\|_F}\nonumber\\
& \leq \underbrace{\|(\bSigma^{\dagger}\widehat{\bSigma} - \bI_{p})\|_F \cdot \| \bB \cdot  \nabla\bF (\boldsymbol{0}_r)^\T\|_F}_{T_1} + \underbrace{\frac{1}{n} \|(\bSigma^{\dagger} -  \widehat{\bSigma}^{\dagger})\|_F \cdot \| \sum_{i = 1}^{n}\left( \bx_i \epsilon_i^\T +  \bx_i\bH(\bB^\T \bx_i)^\T \|\bB^\T\bx_i\|_2  \right)\|_F}_{T_2} 
\end{align}
For term $T_1$ of (\ref{pf-eq-terms}), we have
\begin{align}\label{pf-cov-1}
\|(\bSigma^{\dagger}\widehat{\bSigma} - \bI_{p})\|_F & =  \| \bSigma^{\dagger} (\widehat{\bSigma} - \bSigma)\|_F\nonumber\\
& \leq \|\bSigma^{\dagger}\|_F \|\widehat{\bSigma} - \bSigma\|_F\nonumber\\
& \leq  \frac{\sqrt{r}}{\sigma_r} \cdot \sqrt{r} \|\widehat{\bSigma} - \bSigma\|_2\nonumber\\
& = \frac{r}{\sigma_r} \cdot \|\widehat{\bSigma} - \bSigma\|_2
\end{align}
Then, by Lemma \ref{lemma:covariance}, 
$\|\widehat{\bSigma} - \bSigma\|_2  \leq \max{(\delta, \delta^2)}$
with probability at least $ 1- 2 \exp(-p t^2)$,  where $\delta = CLt\sqrt{p/n}
$, $L = \max{(K,K^2)}$ and  $K = \max_{i}\|\bx_i\|_{\psi_2}$. Let $t = \sqrt{\{\ln(2/\alpha)\}/pd}$, we get
\begin{align*}
\PP\left( \|\widehat{\bSigma} - \bSigma\|_2 \geq CL\sqrt{\frac{\ln(2/\alpha)}{pd}} \right) \leq \alpha,
\end{align*}
it follows that $
\|\widehat{\bSigma} - \bSigma\|_2 = \cO_p(\sqrt{1/n})
$.
Moreover, by Assumption \ref{assum:3}, we have 
\begin{align*}
\| \bB \cdot  \nabla\bF (\boldsymbol{0}_r)^\T\|_F = \cO(\sqrt{pq/r}).
\end{align*}
Thus, the order of term $T_1$ of (\ref{pf-eq-terms}) is $\sqrt{rpq/n}$. For the second term $T_2$ of (\ref{pf-eq-terms}), according to Lemma \ref{lemma:inverse}, we have 
\begin{align*}
\|\bSigma^{\dagger} -  \widehat{\bSigma}^{\dagger}\|_2 \leq \|\bSigma^{\dagger}\|_2 \|\widehat{\bSigma}^{\dagger}\|_2 \|\bSigma^{\dagger} -  \widehat{\bSigma}^{\dagger}\|_2.
\end{align*}
Note that $\|\bSigma^{\dagger}\|_2 = 1/\sigma_r(\widehat{\bSigma}) = 1/\sigma_r^2(\bX/\sqrt{n})$. In addition, by Lemma \ref{lemma:GaussianSingularbound}, we get
\begin{align*}
\sigma_r(\bX/\sqrt{n}) \geq 1 - tL \sqrt{p/n},
\end{align*}
with probability at least $1 - 2p \cdot \exp(-ct^2)$ for some $L >0$. Let $t = \sqrt{\frac{1}{c}\ln(\frac{2p}{\alpha})}$ and we have
\begin{align*}
\PP\left(\sigma_r(\bX/\sqrt{n}) \geq 1- L\sqrt{\frac{p}{cn}\ln(\frac{2p}{\alpha}}) \right) \leq \alpha.
\end{align*}
Hence, we obtain that $\sigma_r(\bX/\sqrt{n})$ is of the order $\cO_p(1)$. Consequently, 
\begin{align}\label{pf-T2-1}
\|\bSigma^{\dagger} -  \widehat{\bSigma}^{\dagger}\|_2 = \cO_p(1/\sqrt{n}) .
\end{align}
Furthermore, 
\begin{align*}
\frac{1}{n}\left\| \sum_{i=1}^n\left( \bx_i \epsilon_i^\T +  \bx_i\bH(\bB^\T \bx_i)^\T \|\bB^\T\bx_i\|_2  \right)\right\|_F \leq \frac{1}{n}  \left(\|\sum_{i=1}^n 
 \bx_i \bepsilon_i^\T\|_F + \|\sum_{i=1}^n  \bx_i\bH(\bB^\T \bx_i)^\T \|\bB^\T\bx_i\|_2  \|_F\right)
\end{align*}
where $\bE := (\bepsilon_1,\ldots,\bepsilon_n)^\T \in \RR^{n\times q}$.
According to Lemma \ref{lemma:GaussianNorm}, we have 
\begin{align}\label{pf-T2-2}
\frac{1}{n}\|\sum_{i=1}^n \bx_i \bepsilon_i^\T\|_F \leq \max_i \|\bx_i\|_2 \|\bepsilon_i\|_2 = \cO_p(\sqrt{pq})
\end{align}
Similarly, 
\begin{align}\label{pf-T2-3}
\frac{1}{n}\|\sum_{i=1}^n  \bx_i\bH(\bB^\T \bx_i)^\T \|\bB^\T\bx_i\|_2  \|_F = \cO_p(\sqrt{pq})
\end{align}
Combining (\ref{pf-T2-1}), (\ref{pf-T2-2}) and (\ref{pf-T2-3}), we achieve the term $T_2$ is of the order $\cO_p(\sqrt{rpq/n})$. In summary, we have
\begin{align}\label{pf-T2-4}
\left\|\frac{1}{n}\sum_{i = 1}^{n} \bs(\bx_i) \by_i^\T - \frac{1}{n}\sum_{i = 1}^{n} \hat{\bs}(\bx_i)  \by_i^\T\right\|_F = \cO_p(\sqrt{\frac{rpq}{n}})
\end{align}
Therefore, by the Theorem 2 in  \cite{yu2015useful} and Theorem 19 in \cite{o2018random}, there exists an orthogonal matrix $\check{\bV} \in \RR^{r\times r}$, such that  
\begin{align}\label{pf-T2-5}
\|\widehat{\bB} -\check{\bB}\check{\bV}\|_F & \leq  2^{\frac{3}{2}} \cdot \frac{\left\| \frac{1}{n}\sum_{i = 1}^{n} \hat{\bs}(\bx_i)\by_i^\T  - \frac{1}{n}\sum_{i = 1}^{n} \bs(\bx_i)  \by_i^\T \right\|_F}{\sigma_r\{\frac{1}{n}\sum_{i = 1}^{n} \bs(\bx_i)\by_i^\T \}} 
\end{align}
Furthermore, let us investigate the order of $\sigma_r\{\frac{1}{n}\sum_{i = 1}^{n} \bs(\bx_i)\by_i^\T\}$. According to Lemma \ref{lemma:1}, with probability at least $1- \delta$, we have
\begin{align*}
    \left \|\frac{1}{n}\sum_{i=1}^{n} \bs(\bx_i) \by_i - \EE\{\bs(\bx) \by^\T\} \right \|_F = \cO\left(\sqrt{\frac{pq}{n}}\sqrt{\ln\frac{pq}{\delta}}\right),
\end{align*}
%
In addition, according to Wely's inequality for singular values, we have 
\begin{align}\label{pf-T2-6}
\left|\sigma_r \left \{\frac{1}{n}\sum_{i=1}^{n}  \bs(\bx_i) \by^{\T}_i \right\} - \sigma_r \left[\EE\{\bs(\bx)\by^\T \}\right] \right| \leq 
\sigma_1 \left [\frac{1}{n}\sum_{i=1}^{n} \bs(\bx_i)  \by^{\T}_i - \EE\{\bs(\bx) \by^\T\} \right] = \cO_p\left(\sqrt{\frac{pq}{n}}\sqrt{\ln\frac{pq}{\delta}}\right)
\end{align}
In summary, with the Assumption \ref{assum:4}, the (\ref{pf-T2-5}) becomes
\begin{align}\label{pf-T2-7}
\|\widehat{\bB}-\check{\bB}\check{\bV} \|_F & \leq  2^{\frac{3}{2}} \cdot \frac{\left\| \frac{1}{n}\sum_{i = 1}^{n} \hat{\bs}(\bx_i) \by_i^\T - \frac{1}{n}\sum_{i = 1}^{n} \bs(\bx_i)\by_i^\T \right\|_F}{\sigma_r\{\frac{1}{n}\sum_{i = 1}^{n} \bs(\bx_i)\by_i^\T\}}\\
& = \cO\left(\sqrt{\frac{r^2p}{n}}\right).
\end{align}
with probability approaching to 1.
\end{proof}

\subsection{Proofs of the Second-order Method}
In the following proofs, absolute constants $C_{5}$ to $C_8$ come from Assumptions~\ref{ass:second-order-subexp-1} to~\ref{ass:second-order-subexp-3}, and $C_{2}$ in Assumption~\ref{ass:second-order-subexp-4}, respectively. 
\subsubsection{Proof of Lemma~\ref{lem:second-order-concentrate}}
\begin{lemma}\label{lem:second-order-concentrate} 
Suppose model~\eqref{model:1.2} holds, under Assumptions~\ref{ass:second-order-subexp-1} to~\ref{ass:second-order-subexp-3}, with probability at least $1 - \delta$, the following holds:
\begin{align*}
   \left \| \frac{1}{n q}\sum^{n}_{i=1}\sum^{q}_{j=1}[\by_{i,j}\bT(\bx_{i}) - \EE\{\by_{j} \bT(x)\} ]\right\|_{F}
   =  \cO \left\{\frac{p}{\sqrt{n}}\sqrt{\ln\left(\frac{p}{\delta}\right)}\right\}.    
\end{align*}
\end{lemma}
\begin{proof}
We first consider $(1/q)\sum_{j=1}^{q}\by_{j} T_{k,l}(\bx)$.
For notational convenience, let $f_{j}$ denote $f_{j}(\bB^{\T}\bx)$. 

\begin{align*}
& \left\{\EE \left |\frac{1}{q}\sum_{j=1}^{q}\by_{j}T_{k,l}(\bx) - \EE \left (\frac{1}{q}\sum_{j=1}^{q}\by_{j}T_{k,l}(\bx) \right ) \right|^{\ell} \right \}^{1/\ell}\\
= & \frac{1}{q} \left\{\EE \left | \sum_{j=1}^{q} 
(f_{j} + \epsilon_{j})  T_{k,l}(\bx)  - \sum_{j=1}^{q}  \EE \left \{(f_{j} + \epsilon_{j}) T_{k,l}(\bx) \right \} \right|^{\ell} \right \}^{1/\ell}\\
\leq & \frac{1}{q} \left \{ \left(\EE\left |\sum_{j=1}^{q} f_{j} T_{k,l}(\bx) \right |^{\ell}\right)^{1/\ell} + \left(\EE\left |\sum_{j=1}^{q} 
\epsilon_{j} T_{k,l}(\bx)\right |^{\ell}\right )^{1/\ell}  \right \} +  \frac{1}{q}    \EE \left [\left|\sum^{q}_{j=1}f_{j} T_{k,l}(\bx) \right |\right ] \\
\leq &\frac{1}{q} \left [ \left \{ \EE\left ( \left |\sum_{j=1}^{q}f_{j}\right |^{2\ell}\right )\EE\left (|T_{k,l}(\bx)|^{2\ell}\right ) \right \}^{1/2\ell} + \left \{\EE\left( \left |\sum_{j=1}^{q}\epsilon_{j}\right|^{2\ell}\right ) \EE\left (|T_{k,l}(\bx)|^{2\ell}\right ) \right \}^{1/2\ell} \right]\\
& + \EE \left (\left|\sum^{q}_{j=1} f_{j} \right |^{2} \right )^{1/2} [\EE \{|T_{k,l}(\bx) |^{2}\}]^{1/2}\\
\leq & \frac{1}{q}\left\{ 2 \left \|\sum^{q}_{j=1} f_j\right \|_{\infty} C_{5}\ell + 2 \left \|\sum^{q}_{j = 1} \epsilon_j \right \|_{\infty} C_{5}\ell + 2 \left\|\sum^{q}_{j=1} f_j \right \|_{2} C_{5}\right \}\\
\leq & \frac{4 + 2\frac{C_{7}}{C_{8}} }{q} \|\sum^{q}_{j=1} f_j \|_{\infty} C_{5}\ell\\
\leq & (4 + 2\frac{C_{7}}{C_{8}}) C_{5} C_{6}\ell. 
\end{align*}

Therefore, for $\forall k,l \in [p]$, $1/q\sum_{j=1}^{q}[\by_{j}T_{k,l}(\bx) - \EE\{\by_{j}T_{k,l}(\bx)\}]$ follows a centered sub-exponential distribution, and $\|1/q\sum_{j=1}^{q}[\by_{j}T_{k,l}(\bx) - \EE\{\by_{j}T_{k,l}(\bx)\}]\|_{\psi_{1}} = \sup_{\ell \geq 1}\frac{1}{\ell} \EE ( | 1/q\sum_{j=1}^{q}[\by_{j}T_{k,l}(\bx) - \EE\{\by_{j}T_{k,l}(\bx)\}]|^{\ell})^{1/\ell} \leq  (4 + 2C_{7}/C_{8}) C_{5} C_{6}$. By Bernstein's inequality (Corollary 2.8.3 in~\citet{vershynin2018}), for $\forall t > 0, k,l \in [p]$,
\begin{align*}
   \PP \left\{\left|\frac{1}{n}\sum_{i = 1}^{n} \left(\frac{1}{q}\sum_{j=1}^{q} [\by_{i,j} T_{k,l}(\bx_i) - \EE \{ \by_{j} T_{k,l}(\bx)\}]\right)\right|\geq t \right\} \leq 2\cdot \exp\left[ -cn \min \left( \frac{t^2}{K^2}, \frac{t}{K}\right)\right],
\end{align*}
where $K \leq  (4 + 2C_{7}/C_{8}) C_{5} C_{6}$ and $c>0$ is an absolute constant. Then, 
\begin{align*}
&\PP \left(\left\|\frac{1}{nq}\sum_{i = 1}^{n} \sum_{j=1}^{q} [\by_{i,j} \bT(\bx_i) - \EE \{ \by_{j} \bT(\bx)\}]\right\|_{F}\geq pt \right) \\
\leq & \sum_{1\leq K \leq l\leq p} \PP \left\{\left|\frac{1}{n}\sum_{i = 1}^{n} \left(\frac{1}{q}\sum_{j=1}^{q} [\by_{i,j} T_{k,l}(\bx_i) - \EE \{ \by_{j} T_{k,l}(\bx)\}]\right)\right|\geq t \right\} \\ 
\leq & p(1 + p) \cdot \exp\left[ -cn \min \left( \frac{t^2}{K^2}, \frac{t}{K}\right)\right] 
\end{align*}

For $\forall \delta >0$, $\exists n$ large enough so that $[\ln\{p(1 + p)/\delta\}]/c/n < 1$. Let $t = K \sqrt{[\ln\{p(1 + p)/\delta\}]/c/n}$, with probability at least $1 - \delta$, 
\begin{align*}
\left\|\frac{1}{nq}\sum_{i = 1}^{n} \sum_{j=1}^{q} [\by_{i,j} \bT(\bx_i) - \EE \{ \by_{j} \bT(\bx)\}]\right\|_{F}\leq pt \leq pK \sqrt{\frac{\ln\left\{\frac{p(p+1)}{\delta}\right\}}{cn}},  
\end{align*}
i.e., with probability at least $1 - \delta$, 
\begin{align*}
\left\|\frac{1}{nq}\sum_{i = 1}^{n} \sum_{j=1}^{q} [\by_{i,j} \bT(\bx_i) - \EE \{ \by_{j} \bT(\bx)\}]\right\|_{F} = \cO \left(p\sqrt{\frac{\ln\left(\frac{p}{\delta}\right)}{n}}\right).  
\end{align*}

\end{proof}

\subsubsection{Proof of Theorem~\ref{thm:second-order}}
\begin{proof}
By Lemma~\ref{lem:second-order-concentrate}, with probability at least $1 - \delta$,
\begin{align*}
\left\|\frac{1}{nq}\sum_{i = 1}^{n} \sum_{j=1}^{q} [\by_{i,j} \bT(\bx_i) - \EE \{ \by_{j} \bT(\bx)\}]\right\|_{F} = \cO \left(p\sqrt{\frac{\ln\left(\frac{p}{\delta}\right)}{n}}\right).  
\end{align*}

Moreover, by Lemma~\ref{lem:ss} and Assumption~\ref{ass:second-order-subexp-4}, $\lambda_{r}[1/q \sum_{j=1}^{q}  \EE \{ \by_{j} \bT(\bx)\}] = \lambda_{r}[1/q\sum_{j=1}^{q} \EE \{ \nabla^{2} f_{j}(\bB^{\T}\bx) \}] = \Omega(1/r)$. By Theorem 2 in  \cite{yu2015useful} there exists an orthogonal matrix $\hat{\bV} \in \RR^{r\times r}$, such that 
\begin{align*}
\|\bB-\Tilde{\bB} \hat{\bV}\|_F \leq 2^{\frac{3}{2}}\frac{
\left\|\frac{1}{nq}\sum_{i = 1}^{n} \sum_{j=1}^{q} [\by_{i,j} \bT(\bx_i) - \EE \{ \by_{j} \bT(\bx)\}]\right\|_{F}}{\lambda_{r}[1/q \sum_{j=1}^{q}  \EE \{ \by_{j} \bT(\bx)\}]}.
\end{align*}
Therefore, with probability at least $1 - \delta$, 
\begin{align*}
\inf_{\bV\in\mathcal{V}_r}\|\bB-\Tilde{\bB} \bV\|_F =   \cO \left(p r\sqrt{\frac{\ln\left(\frac{p}{\delta}\right)}{n}}\right). 
\end{align*}
\end{proof}


\end{document}